\documentclass[11pt]{article}
\usepackage{amsmath,amssymb,amsfonts,amsthm,epsfig}
\usepackage[usenames,dvipsnames]{xcolor}
\usepackage{bm,xspace}
\usepackage{tcolorbox}
\usepackage{cancel}
\usepackage{fullpage}
\usepackage{tabularx}
\usepackage{booktabs}
\usepackage{guy}
\usepackage{framed}
\usepackage{verbatim}
\usepackage{enumitem}
\usepackage{array}
\usepackage{multirow}
\usepackage{afterpage}
\usepackage{mathrsfs}
\usepackage{pifont} 
\usepackage{chngpage}
\usepackage[normalem]{ulem}
\usepackage{boxedminipage}
\usepackage{caption}
\usepackage{forest}
\usepackage{svg}
\usepackage{microtype}

\usepackage{algorithm}
\usepackage{algorithmicx}
\usepackage{algpseudocode}

\usepackage{pgfplots}
\pgfplotsset{width=8cm,compat=newest}

\def\colorful{1}

\ifnum\colorful=1

\fi
\ifnum\colorful=0

\fi

\newcommand{\base}{\mathrm{base}}

\newcommand{\error}{\mathrm{error}}

\newcommand{\cost}{\mathrm{cost}}
\newcommand{\phin}{\phi^{(n)}}

\newcommand{\KL}[2]{d_{\mathrm{KL}}\left({#1}\middle\|{#2}\right)}
\newcommand{\KLbig}[2]{d_{\mathrm{KL}}\big({#1}\big\|{#2}\big)}

\newcommand{\Unif}{\mathrm{Unif}}
\newcommand{\Ber}{\mathrm{Ber}}
\newcommand{\Bin}{\mathrm{Bin}}

\newcommand{\AG}{\mathrm{AG}}
\newcommand{\stab}{\mathrm{stab}}

\DeclareMathOperator*{\argmax}{arg\,max}

\newlist{enumprop}{enumerate}{1} 
\setlist[enumprop]{label=\arabic*.,ref=\theproposition.\arabic*}
\crefalias{enumpropi}{proposition}

\newcommand{\pparagraph}[1]{\bigskip \noindent {\bf {#1}}}

\begin{document}

\title{
Subsampling Suffices for Adaptive Data Analysis
}

\author{Guy Blanc \vspace{8pt} \\ \hspace{-5pt}{\sl Stanford University}}

\date{\vspace{15pt}\small{\today}}

\maketitle

\begin{abstract}
    Ensuring that analyses performed on a dataset are representative of the entire population is one of the central problems in statistics. Most classical techniques assume that the dataset is independent of the analyst's query and break down in the common setting where a dataset is reused for multiple, adaptively chosen, queries. This problem of \emph{adaptive data analysis} was formalized in the seminal works of Dwork et al. (STOC, 2015) and Hardt and Ullman (FOCS, 2014).

We identify a remarkably simple set of assumptions under which the queries will continue to be representative even when chosen adaptively: The only requirements are that each query takes as input a random subsample and outputs few bits. This result shows that the noise inherent in subsampling is sufficient to guarantee that query responses generalize. The simplicity of this subsampling-based framework allows it to model a variety of real-world scenarios not covered by prior work.

In addition to its simplicity, we demonstrate the utility of this framework by designing mechanisms for two foundational tasks, statistical queries and median finding. In particular, our mechanism for answering the broadly applicable class of statistical queries is both extremely simple and state of the art in many parameter regimes.
\end{abstract}

\thispagestyle{empty}
\newpage 
\setcounter{page}{1}

\tableofcontents 
\newpage
\section{Introduction}

Data is a scarce and valuable resource. As a result, data analysts often reuse the same dataset to answer multiple queries. Hardt and Ullman \cite{HU14} as well as Dwork, Feldman, Hardt, Pitassi, Reingold, and Roth \cite{DFHPRR15} initiated the study of \emph{adaptive data analysis} which aims to give provable guarantees that the query answers will have low bias, i.e. be representative of the full population, even when a dataset is reused adaptively. Since then, there have been a number of works exploring the adaptive reuse of data \cite{DFHPR15,DFHPR15b,SU15,SU15between,BNSSSU16,RZ16,RRST16,smith2017survey,FS17,FS18,FRR20,DK22}.

Much prior work has focused on the design of \emph{mechanisms}, a layer between the analyst and dataset. In those works, analysts do not have direct access to the data. Instead, when they wish to ask a query $\phi$, they pass it to the mechanism. The mechanism answers $\phi$ without revealing too much information about the dataset; e.g. by adding noise to the output of $\phi$ applied to the dataset. \cite{DFHPR15,DFHPR15b,DK22,FS18,FS17,SU15between}.

This work is motivated by a simple question.
\begin{quote}
    \emph{How can we guarantee that the query responses will have low bias, even without an explicit mechanism?}
\end{quote}
The purpose of asking this question is twofold. First, it models a reality in which data analysts often do not explicitly use mechanisms to obfuscate query results before looking at them. Second, if the assumptions we make on the queries are sufficiently easy to explain to an analyst, they are actionable, as the analyst can keep these assumptions in mind when deciding how to analyze data.

\pparagraph{Our approach.} We show that as long as each query takes as input a random subsample of the dataset and outputs to a small range, the results will have low bias. Quantitatively, our results depend \emph{only} on the size of the subsample and the number of possible outputs the query has. These quantities are intuitive and easy to explain to a data analyst who may be interested in ensuring their results have low bias. Furthermore many algorithms, such as stochastic gradient descent, already subsample their data, in which case we guarantee generalization for free.

One interpretation of this framework is that it eliminates the need to design a noise distribution for each task. Prior works design mechanisms to bound the bias by adding an appropriate amount of noise to the true result before returning it (e.g. by adding a mean-$0$ Gaussian). Our work shows that the noise inherent in subsampling suffices. It also extends to tasks where it is difficult to design an appropriate noise distribution -- for example, when the output of each query is categorical rather than numerical.

As easy corollaries of this subsampling approach, we give simple mechanisms for two foundational tasks, statistical queries and median finding, demonstrating the power of this framework. In particular, our mechanism for answering the broad and influential class of \emph{statistical queries} (SQs) \cite{Kea98,FGRVX17} achieves state of the art accuracy in many parameter regimes. In addition to their broad applicability, statistical queries have been the standard bearer by which we assess the utility of approaches for adaptive data analysis since the early works of \cite{DFHPRR15,BNSSSU16}.  Our SQ mechanism had advantages beyond its accuracy: It runs in sublinear time, and its extreme simplicity renders it broadly applicable in non-standard setting.

\subsection{Main results}

We show that an analyst can ask an adaptive sequence of \emph{subsampling queries} without incurring large bias.

\begin{definition}[Subsampling query]
    \label{def:subsampling-query}
    For any sample $S \in X^n$ and function $\phi:X^w \to Y$, the \emph{subsampling query} $\phi$ is answered by drawing $\bx_1, \ldots, \bx_w$ uniformly without replacement from $S$, and then providing the response $\by = \phi(\bx_1, \ldots, \bx_w)$.

    In an abuse of notation, we use $\phi(S)$ to denote the distribution of $\by$ defined above. Similarly, for any distribution $\mcD$ supported on $X$, we use $\phi(\mcD)$ to denote the distribution of $\by' = \phi(\bx_1', \ldots, \bx_w')$ when $\bx_1', \ldots, \bx_w' \iid \mcD$.
\end{definition}

We allow the analyst's choice of queries to be \emph{adaptive}: Formally, at each time step $t \in [T] \coloneqq \set{1,2,\ldots, T}$, the analyst selects a query $\phi_t:X^{w_t} \to Y_t$ which may depend on the previous responses $\by_1, \ldots, \by_{t-1}$ and then receives the response $\by_t \sim \phi_t(S)$. We begin with an informal result bounding the sample size needed to ensure the results have low bias.\medskip
\begin{tcolorbox}[colback = white,arc=1mm, boxrule=0.25mm]
\begin{theorem}[Subsampling responses have low bias]
    \label{thm:informal}
    Suppose an analyst asks an adaptive sequence of $T$ subsampling queries, each mapping $X^w$ to $Y$, to a sample $\bS \sim \mcD^n$. As long as
    \begin{equation*}
        n \geq \Omega(w\sqrt{T|Y|}),
    \end{equation*}
    with high probability, all of the queries will have low bias.
\end{theorem}
\end{tcolorbox}
\medskip

Compare \Cref{thm:informal} to a naive approach which takes a fresh batch of $w$ samples for each query. Subsampling has a quadratically better dependence on the number of queries asked, $T$, than that naive approach which requires $n \geq wT$. Based on the lower bounds in \cite{HU14,SU15}, this quadratic improvement is optimal.

Informally, a subsampling query $\phi$ has low bias if the distributions $\phi(\bS)$ and $\phi(\mcD)$ are ``close." We'll formalize and generalize \Cref{thm:informal} in \Cref{sec:tech-overview-bias}. That generalization will, for example, allow the analyst to choose a different domain size and range for each query. 

The main ingredient in the proof of \Cref{thm:informal} is a bound on the mutual information between the sample and the query responses.

\begin{theorem}[Subsampling queries reveal little information]
    \label{thm:MI-intro}
    Suppose an analyst asks an adaptive sequence of $T$ subsampling queries to a sample $\bS \sim \mcD^n$ where the $t^{\text{th}}$ query takes as input a subsample of size $\bw_t$ and outputs to a range of size $\br_t$. Then, the mutual information between the query responses and the sample is at most
    \begin{equation*}
        I(\bS;(\by_1, \ldots, \by_T)  )\leq \frac{4\Ex\bracket*{\sum_{t=1}^T \bw_t (\br_{t}-1)}}{n}
    \end{equation*}
    where $\by_1, \ldots ,\by_T$ are the query responses and the expectation is over the analyst's choices for the subsampling queries.
\end{theorem}


\subsection{Applications}
\label{sec:applications-intro}
\subsubsection{A mechanism for statistical queries}
\label{subsec:SQ}

Our main application is an extremely simple and accurate mechanism for the broad class of \emph{statistical queries}. Statistical queries, introduced by Kearns \cite{Kea98}, are parameterized by a function $\phi: X \to [0,1]$. A valid answer to such a query is any value close to $\mu_{\phi}(\mcD) \coloneqq \Ex_{\bx \sim \mcD}[\phi(\bx)]$. Many natural analyses can be cast as sequence of statistical queries. This includes most algorithms for both supervised and unsupervised learning, such as least-squares regression, gradient-descent, and moment based methods. We refer the reader to \cite{FGRVX17} for a more extensive list. 
\begin{figure}[ht] 

  \captionsetup{width=.9\linewidth}

    \begin{tcolorbox}[colback = white,arc=1mm, boxrule=0.25mm]
    \vspace{2pt} 
    
    \textbf{Input:} A sample $S \in X^n$ and parameter $k$.\vspace{2pt}
    
    For each time step $t \in [T] \coloneqq \set{1,2,\ldots, T}$,\vspace{2pt}
    \begin{enumerate}[nolistsep,itemsep=2pt]
        \item Receive a query $\phi_t: X\to [0,1]$.
        \item Sample $\bx_1, \ldots, \bx_k \iid \Unif(S)$ and, for each, sample a vote $\bv_i \sim \mathrm{Bernoulli}(\phi_t(\bx_i))$.
        \item Output $\by_t \coloneqq (\bv_1 + \cdots + \bv_k)/k$.
    \end{enumerate}
    
    \end{tcolorbox}
\caption{A mechanism for answering statistical queries using subsampling.}
\label{fig:SQ-mechanism}
\end{figure}





\begin{theorem}[Accuracy of our mechanism for answering SQs]
    \label{thm:SQ}
    For any $\tau, \delta > 0$ and adaptively chosen sequence of statistical queries $\phi_1, \ldots, \phi_T: X \to [0,1]$, with parameters
    \begin{equation}
        \label{eq:SQ-sample-size}
        n \coloneqq O\paren*{\frac{\sqrt{T \ln(T/\delta) \ln(1/\delta)}}{\tau^2}} \quad\quad\quad\quad\text{and}\quad\quad\quad\quad k \coloneqq \Theta\paren*{\frac{\ln(T/\delta)}{\tau^2}},
    \end{equation}
    the mechanism of \Cref{fig:SQ-mechanism}, when given a sample $\bS \sim \mcD^n$ and parameter $k$, with probability at least $1 - \delta$, answers all queries $t \in [T]$ to accuracy
    \begin{equation}
        \label{eq:SQ-accuracy}
        |\by_t - \mu_{\phi}(\mcD)| \leq  \max(\tau \cdot \std_{\phi_t}(\mcD), \tau^2)
    \end{equation}
    where $\std_{\phi}(\mcD) \coloneqq \sqrt{\mu_{\phi}(\mcD)(1 - \mu_{\phi}(\mcD))} \leq 1$.
\end{theorem}
The proof of \Cref{thm:SQ} is a simple corollary of our subsampling framework: Each vote ($\bv_i$ in \Cref{fig:SQ-mechanism}) is the output of a subsampling query $\phi:X \to \zo$ and so fits within our framework with $w = 1$, and $|Y| = 2$.

In many settings, the bound of \Cref{thm:SQ} is state of the art. When $\std_{\phi}(\mcD) = o(1)$, its accuracy improves upon prior state of the arts, both from Feldman and Steinke.\footnote{It's worth noting that both of these works use a more strict definition of $\std_{\phi}'(\mcD) = \sqrt{\Var_{\bx \sim \mcD}[\phi(\bx)]}$. If the range of $\phi$ is $\zo$, their definition and ours coincide. Otherwise, their definition can be smaller than ours.}
\begin{enumerate}
    \item In \cite{FS17}, they gave a mechanism with accuracy guarantees similar to\footnote{In particular, when $\std_{\phi}(\mcD)$ is very small, the accuracy guarantee in \cite{FS17} can improve on that of \Cref{thm:SQ}. See the discussion in \Cref{subsec:approx-median} on using approximate median mechanisms to answer SQs for details on how the accuracy guarantees compare.} \Cref{thm:SQ}, but requiring a larger sample size of
    \begin{equation*}
        n \geq \Omega\paren*{\frac{\sqrt{T \ln(1/\delta)} \ln(T/\tau\delta)}{\tau^2}}.
    \end{equation*}
    In addition to requiring a larger sample size than that of \Cref{thm:SQ}, that mechanism is also more complicated than ours. It starts by splitting the dataset into chunks each containing $\frac{1}{\tau^2}$ points. When given a query, it first computes the average of that query on each chunk, and then computes an approximate median of those averages via a differentially private algorithm. The same mechanism actually solves the approximate median problem described in \Cref{subsec:approx-median}.
    \item In \cite{FS18}, they gave a mechanism with a slightly worse sample size\footnote{Their sample size is a $\sqrt{\ln(T)}$ multiplicative factor worse than ours even when $\delta$ is constant.} than that of \Cref{thm:SQ} when the failure probability, $\delta$, is a constant. Their mechanism is also simple: For a sample $S$ and query $\phi$, they compute $\phi(S) + \zeta$ where $\zeta$ is a Gaussian with mean $0$ and variance that scales with $\std_{\phi}(\mcD)$. However, their dependence on $\delta$ is a multiplicative $1/\delta$ factor, exponentially worse than ours.
\end{enumerate}

The non-variance-adaptive setting, where we only wish to bound $|\by_t - \mu_{\phi}(\mcD)| \leq   \tau^2$ regardless of $\std_{\phi}(\mcD)$, is more well studied \cite{DFHPRR15,BNSSSU16,SU15,SU15between,HU14,DK22}. The state of the art was given recently by Dagan and Kur \cite{DK22}, who showed that a sample of size
\begin{equation*}
    n = O\paren*{\frac{\sqrt{T\ln(1/\tau \delta)}}{\tau^2}}
\end{equation*}
is sufficient in most parameter regimes. Their mechanism works by returning $y_t = \phi_t(S) + \zeta$ where $\zeta$ is drawn from a very carefully constructed \emph{bounded} distribution. Our mechanism has a slightly better dependence on $\tau$ and a better accuracy when $\std_{\phi}(\mcD)$ is small, but slightly worse dependencies on $T$ and $\delta$.

\paragraph{Advantages of our mechanism.} 
Our mechanism has advantages beyond the quantitative bound on the sample size needed for low bias. First, it naturally runs in sublinear time as answering each query requires only looking at $k \approx n/\sqrt{T}$ of the points in $S$. 

Furthermore, our mechanism easily extends to the setting where the analyst does not know ahead of time how many samples, the parameter $k$ in \Cref{fig:SQ-mechanism}, they want for a particular query. Rather, they can sequentially sample votes $\bv_1, \bv_2, \ldots$ while continually updating an estimate for $\mu_{\phi}(\mcD)$ and stop at any point. The bounds of \Cref{thm:SQ} hold as long as the total number of votes, summed over all queries, is not too large. Early stopping can appear naturally in practice. For example,
\begin{enumerate}
    \item The analyst may only desire accuracy $\pm \tau$, regardless of what $\std_{\phi}(\mcD)$ is. If, based on the first $k' < k$ samples, the analyst can determine that $\std_{\phi}(\mcD)$ is small, they can stop early as a small $\std_{\phi}(\mcD)$ means fewer samples are needed to achieve the desired accuracy.
    \item The analyst may wish to verify whether $\mu_{\phi}(\mcD) \approx c$ for some value $c$. If after the first $k' < k$ samples, the average is far from $c$, the analyst can already determine that $\mu_{\phi}(\mcD)$ is far from $c$. This setting has previously been studied in the influential work of \cite{DFHPR15,DFHPR15b} which showed that there exists mechanisms that answer exponentially many such verification queries, as long as all but a tiny fraction of the inequalities to verify are true. Our analysis does not extend to exponentially many queries, but it can easily intertwine standard queries with verification queries, with the later being cheaper.
\end{enumerate}

\subsubsection{A mechanism for finding approximate medians}
\label{subsec:approx-median}
We also consider a generalization of statistical queries, each of which map $w$ inputs to some set $R \subseteq \R$. For such queries, we give a mechanism for determining an \emph{approximate median} of the distribution $\phi(\mcD)$.
\begin{definition}[Approximate median]
    \label{def:approx-median}
    For a distribution $\mcE$ with support in $\R$, we say a value $y$ is an \emph{approximate median of $\mcE$} if,
    \begin{equation*}
        \min\left(\Pr_{\bx \sim \mcE}[\bx \leq y], \Pr_{\bx \sim \mcE}[\bx \geq y]\right) \geq 0.4.\footnote{All of our results hold as long as this $0.4$ is $0.5 - \eps$ for any fixed $\eps > 0$.}
    \end{equation*}
\end{definition}

One particular application for approximate median queries is, once again, for answering statistical queries. Given an SQ $\phi: X \to [0,1]$, we can construct
\begin{equation*}
    \phi'(x_1, \ldots, x_w) = \Ex_{\bi \in [w]}[\phi(\bx_i)].
\end{equation*}
Since $\phi'$ and $\phi$ have the same mean, and $\phi'$ has a smaller variance, Chebyshev's inequality implies that any approximate median of $\phi'$ will be within $2/\sqrt{w}$ standard deviations of $\mu_{\phi}(\mcD)$. As a result, the mechanism of \Cref{fig:median-mechanism} can give similar accuracy results as guaranteed in \Cref{thm:SQ}. The sample size required is larger (by log factors), but in exchange, it provides better accuracy when $\std_{\phi}(\mcD)$ is \emph{very} small. In particular, when $\std_{\phi}(\mcD) < \tau$, the $\tau^2$ term of \Cref{eq:SQ-accuracy} dominates, but the median-based mechanism does not incur it.

\begin{figure}[h] 

  \captionsetup{width=.9\linewidth}

    \begin{tcolorbox}[colback = white,arc=1mm, boxrule=0.25mm]
    \vspace{2pt} 
    
    \textbf{Input:} A sample $S \in X^n$ split into $k$ disjoint groups $S_1, \ldots, S_k$ each containing $\floor{n/k}$ elements (ignoring leftover elements if $n$ is not a multiple of $k$). \vspace{6pt}
    
    For each time step $t \in [T]$,\vspace{2pt}
    \begin{enumerate}[nolistsep,itemsep=2pt]
        \item Receive a query $\phi_t: X^{w_t}\to R_t$ where $R_t \subseteq \R$.
        \item Perform binary search on the mechanism's output $\by_t \in R_t$ where, to determine whether $\by_t \geq r$, the following procedure is used.
        \begin{enumerate}
            \item For each group $i \in [k]$, sample $\bz_i \sim \phi_t(S_i)$ and set a vote to $\bv_i \coloneqq \Ind[\bz_i \geq r]$.
            \item Determine $\by_t \geq r$ iff at least half of the votes are $1$.
        \end{enumerate}
        \item After $\ceil{\log_2(|R_t|)}$ steps of binary search have finished, a single value, $\by_t = r$, will be determined. Output it.
    \end{enumerate}
    
    \end{tcolorbox}
\caption{The subsampling mechanism answering approximate median queries.}
\label{fig:median-mechanism}
\end{figure} 

\begin{theorem}[Accuracy of our mechanism for answering approximate median queries]
    \label{thm:median}
    For any $\delta > 0$, adaptively chosen sequence of queries $\phi_1: X^{w_1} \to R_1 , \ldots, \phi_T: X^{w_T} \to R_T$, and parameters set to
    \begin{equation*}
       n \coloneqq O\paren*{\ln\paren*{\frac{T \ln R_{\max}}{\delta}} \sqrt{ w_{\max} \sum_{t \in [T]} w_t \cdot \ln |R_t|} }  \quad\quad\quad\quad\text{and}\quad\quad\quad\quad k \coloneqq O\paren*{\ln\paren*{\frac{T \ln R_{\max}}{\delta}}}
    \end{equation*}
    where $w_{\max}$ and $R_{\max}$ are upper bounds on $w_t$ and $|R_t|$ respectively, the mechanism of \Cref{fig:median-mechanism}, when given a sample $\bS \sim \mcD^n$ and parameter $k$, with probability at least $1 - \delta$, successfully outputs $\by_t$ that is an approximate median of $\phi_t(\mcD)$ for all $t \in [T]$.
\end{theorem}

\begin{remark}[Design decisions for the mechanism in \Cref{fig:median-mechanism}]
    The mechanism for answering approximate median queries in \Cref{fig:median-mechanism} is a compromise between accuracy and simplicity.
    \begin{enumerate}
        \item It could have been made simpler by not ``grouping" $S$ into $k$ disjoint groups. In this variant, whenever the mechanism wants a vote, it simply samples $\Ind[\bz \geq r]$ where $\bz \sim \phi_t(S)$. In exchange, our techniques would only give an accuracy guarantee that holds in expectation and therefore has a poor dependence on the failure probability, whereas \Cref{thm:median} has only a logarithmic dependence on the failure probability.
        \item Using ``noisy binary search" \cite{KK07} in place of standard binary search would improve the log factors in \Cref{thm:median}. However, this approach leads to a more complicated mechanism.
    \end{enumerate}
\end{remark}

Feldman and Steinke also give a mechanism for answering approximate median queries \cite{FS17}. Their mechanism needs a sample size of
\begin{equation*}
    n \geq \Omega\paren*{\ln\paren*{\frac{T R_{\max}}{\delta}} \sqrt{\ln(1/\delta) T w_{\max}^2} }.
\end{equation*}
Our sample size bound is similar to theirs in the pessimistic settings where $w_t \approx w_{\max}$ for all $t$, with slight improvements on some of the other dependencies. For example, we have a linear dependence on $\ln(1/\delta)$, whereas they have a $\ln(1/\delta)^{3/2}$ dependence.

Most interestingly, our mechanisms and that of \cite{FS17} is fairly similar -- both rely on splitting the dataset and, roughly speaking, computing an approximate median of the queries' value on each group -- but the analyses are wholly different. Their analysis is based on differential privacy. In contrast, \Cref{thm:median} is a simple corollary of our subsampling framework. Indeed, it's not difficult to show that our mechanism does \emph{not} satisfy standard $(\eps, \delta)$-differential privacy with strong enough parameters to give a sample size bound close to that of \Cref{thm:median}.
\subsection{Other related work}
Subsampling has been thoroughly explored in the context of privacy amplification (see e.g. \cite{BBG18,ZW19} or the book chapter \cite{steinkeBookChapter}): if $\mcA$ is a differentially private algorithm, running $\mcA$ on a random subset of the data gives an algorithm with even better privacy parameters. Given the previous applications of differential privacy to adaptive data analysis, this seems like a natural starting point for our work.  However, such an approach is not sufficient to analyze subsampling queries. Indeed, subsampling queries do not necessarily satisfy $(\eps, \delta)$-differential privacy with sufficiently good parameters to give useful bounds on the bias.

Fish, Reyzin, and Rubinstein explored the use of subsampling to speed up classical mechanisms for adaptive data analysis \cite{FRR20}. For example, their mechanism for answering a statistical query $\phi$, computes $\phi$ on a random subsample of the data \emph{and} adds Laplacian noise to that result. This allows them to retain the accuracy guarantees of prior mechanisms that added Laplacian noise \cite{BNSSSU16} while also running in sublinear time. In contrast, our work shows that subsampling alone is sufficient, and achieves sample size bounds that improve upon prior work.

\begin{table}[htb]
\centering
\begin{tabular}{l l}
  \multicolumn{2}{c}{\textbf{Indexing}} \\ 
  \midrule
  $[n]$ for $n \in \N$& The set $\set{1, \ldots, n}$ \\
  $[a,b]$ for $a \leq b \in \N$& The set $\set{a, a+1, \ldots, b}$ \\
  $\binom{S}{k}$ for $S \in X^n$ and $k \in [n]$& All size-$k$ subsets of $S$\\
  $S^{(w)}$ for $S \in X^n$ and $k \in [n]$& All \emph{ordered} tuples containing $k$ element`s from $S$\\
  $S_i$ for $S \in X^n$ and $i \in [n]$ &The $i^{\text{th}}$ element of $S$ \\
  $S_{-i}$ for $S \in X^n$ and $i \in [n]$ &The tuple $(S_1, \ldots, S_{i-1}, S_{i+1}, \ldots S_n)$ \\
  $S_{J}$ for $S \in X^n$ and $J \subseteq [n]$ or $J \in [n]^{(w)}$ &The tuple consisting of $S_i$ for each $i \in J$ \\
  $S_{-J}$ for $S \in X^n$ and $J \subseteq [n]$ or $J \in [n]^{(w)}$ &The tuple consisting of $S_i$ for each $i \notin J$ \\
    \midrule
  \multicolumn{2}{c}{\textbf{Random variables and distributions}} \\ 
    \midrule
  $\bx \sim \mcD$ for distribution $\mcD$  & Random variable $\bx$ is drawn from $\mcD$ \\
  $\bx \sim \Unif(S)$ or $\bx \sim S$ for $S \in X^n$  & $\bx$ is drawn uniformly from $S$ \\
  $\bx_1,\ldots, \bx_k \iid \mcD$ or $\bx \sim \mcD^k$ & $\bx_1,\ldots, \bx_k$ are independently drawn from $\mcD$ \\
  $\mcD(x)$ for distribution $\mcD$ & The probability $\bx = x$ for $\bx \sim \mcD$ \\
  $\Ber(p)$ for $p \in [0,1]$ & The Bernoulli distribution \\
  $\Bin(n,p)$ for $p \in [0,1]$ and $n \in \N$ & The Binomial distribution\\
    \midrule
   \multicolumn{2}{c}{\textbf{Properties of random variables and distributions}} \\
     \midrule
  $\KL{\mcD}{\mcE}$ for distributions $\mcD$ and $\mcE$ & The KL Divergence between $\mcD$ and $\mcE$ \\
  $H(\bx)$ for random variable $\bx$ & The entropy of $\bx$ \\
  $I(\bx; \by)$ for random variables $\bx,\by$ & The mutual information of $\bx$ and $\by$ \\
  $I(\bx; \by \mid \bz)$ for random variables $\bx,\by, \bz$&The mutual information of $\bx,\by$ conditioned on $\bz$\\
\bottomrule
\end{tabular}
\caption{A summary of the notation used. More detailed descriptions of this notation is given in \Cref{sec:prelim}}
\label{table:notation}

\end{table}

\section{Technical overview}
\label{sec:technical-overview}

For the vast majority of mechanisms for adaptive data analysis, differential privacy (DP) is used to prove correctness. This is done in three steps.
\begin{enumerate}
    \item Ensure that each answer the mechanisms provides is $(\eps, \delta)$-DP. This is typically accomplished by adding noise as appropriate.
    \item Show that DP composes nicely, even when the queries are chosen adaptively.
    \item Prove, via a transfer theorem, that any differentially private mechanism also generalizes.
\end{enumerate}
Steps 2 and 3 are, by now, quite generic. See, for example, \cite{DRV10} for composition and \cite{BNSSSU16} for a transfer theorem. Therefore, when designing a new mechanism, one ``only" needs to verify that the mechanism satisfies $(\eps,\delta)$-DP for appropriate $\eps,\delta$. Specifically, if the aim is to answer $T$ many queries, we roughly need $\eps \leq \frac{1}{\sqrt{T}}$ and $\delta \leq \frac{1}{T}$.

Unfortunately, this approach only gives trivial bounds for subsampling queries. Consider the simplest type of subsampling query that takes as input a single point and outputs one bit of information about it (meaning $\phi:X^1 \to \zo$). The output of the subsampling query $\phi(S)$ on sample $S \in X^n$ is distributed according to a Bernoulli with $p = \frac{k}{n}$, where $k$ is the number of points in $S$ on which $\phi$ outputs $1$. Therefore, there are $S, S' \in X^n$ differing in only one point for which $\phi(S) \sim \Ber(0)$ and $\phi(S') \sim \Ber(\frac{1}{n})$. As a result, this subsampling query is $(\eps, \delta)$-DP only if $\delta \geq \frac{1}{n}$, regardless of how $\eps$ is set. Such weak parameters only allow for adaptively answering $\approx n$ queries.

In contrast, \Cref{thm:informal} allows for $\approx n^2$ of these simple queries, which is optimal. To achieve this quadratic improvement over the DP-based bounds, we need a different approach.



\subsection{Sketch of \texorpdfstring{\Cref{thm:informal}}{Theorem \ref{thm:informal}} for the simplest subsampling queries}
\label{subsec:ALKL-overview}
We begin by sketching a proof of our main result in the simplest setting. Specifically, we'll show that if an analyst asks $\tilde{O}(n^2)$ subsampling queries, each mapping $X^1$ to $\zo$, the last query asked will have low bias. This sketch follows the same three steps described in \Cref{sec:technical-overview}, except we replace ``differential privacy" with \emph{average leave-one-out KL stability}.
\begin{definition}[Average leave-one-out KL stability \cite{FS18}]
    \label{def:ALKL-intro}
    A randomized algorithm $\mcM: X^n \to Y$ is $\eps$-\emph{ALOOKL stable} if there is a randomized algorithm $\mcM': X^{n-1} \to Y$ such that, for all samples $S \in X^n$,
    \begin{equation*}
        \Ex_{\bi \sim \Unif([n])} \bracket*{\KL{\mcM(S)}{\mcM'(S_{-{\bi}})}} \leq \eps,
    \end{equation*}
     where $\KL{\cdot}{\cdot}$ is the KL-divergence and $S_{-i}$ refers to the $S$ with the $i^{\text{th}}$ coordinate removed.
\end{definition}
With this definition in hand, we seek to show the following three claims.
\begin{enumerate}
    \item For any $\phi:X^1 \to \zo$, the randomized algorithm mapping $S \in X^n$ to $\phi(S)$ is $\tilde{O}(1/n^2)$-ALOOKL stable.
    \item ALOOKL stability composes linearly and adaptively: Applied to our setting, this means that any analyst that adaptively asks $T$ queries each of which are $\eps$-ALOOKL stable and then chooses a test as a function of the query responses, the algorithm mapping the sample to the test is $(T\eps)$-ALOOKL stable.
    \item ALOOKL stability bounds bias: Applied to our setting, let $\mcM$ be an ALOOKL stable randomized algorithm that takes as input a sample set and outputs some test $\psi:X^1 \to [0,1]$. Then, if we sample $\bS \sim \mcD^n$ and a test function $\bpsi \sim \mcM(\bS)$, the expectation of $\bpsi(\bS)$ is close to the expectation of $\bpsi(\bx)$ where $\bx \sim \mcD$ is a fresh sample independent of $\bpsi$.
\end{enumerate}

The second claim, composition, was proven by Feldman and Steinke in their work introducing ALOOKL stability \cite{FS18}. We therefore only need to sketch the first and third claims.

\pparagraph{Simple subsampling queries are ALOOKL stable.} Fix some subsampling query $\phi:X^1 \to \zo$ and let $\mcM_{\phi}:X^n \to \zo$ be the randomized algorithm that takes in a sample $S$ and outputs $\phi(S)$. The first step in bounding the ALOOKL stability of $\mcM_{\phi}$ is designing the randomized algorithm $\mcM'_{\phi}:X^{n-1} \to \zo$. The randomized algorithm, $\mcM'_{\phi}$, will do the following: With probability $\frac{1}{n^2}$, it outputs a uniform bit in $\zo$. Otherwise, it runs the subsampling query $\phi$ on its input (i.e. outputs $\phi(\bx)$ for $\bx$ chosen uniformly from its input).

Rather than a full analysis proving the stability of $\mcM_{\phi}$, we only sketch two extreme and representative inputs. 
\begin{enumerate}
    \item Consider an input $S \in X^n$ for which $\phi$ outputs $1$ on roughly $n/2$ of the $x \in S$. In this case, $\mcM_{\phi}(S) \sim \Ber(p)$ for $p \approx 1/2$. Furthermore, regardless of how we remove a point from $S$, we'll have that $\mcM_{\phi}(S_{-i}) \sim \Ber(p')$ for $|p' - p| \leq O(1/n)$. The desired bound on ALOOKL stability follows from the observation that $\KL{\Ber(p)}{\Ber(p + O(1/n))} \leq O(1/n^2)$ as long as $p$ is close to $1/2$.
    \item For the other case, consider an input $S \in X^n$ for which $\phi$ outputs $1$ on exactly one $x \in S$ and outputs $0$ on all $n-1$ other points in $S$. On this input, $\mcM_{\phi}(S) \sim \Ber(1/n)$. Then, with probability $\frac{1}{n}$ over $\bi$, $S_{-\bi}$ will contain only points on which $\phi$ evaluates to $0$, in which case, $\mcM_{\phi}'(S_{-\bi}) \sim \Ber(1/(2n^2))$. With the remaining $1 - \frac{1}{n}$ probability over $\bi$, $S_{-\bi}$ outputs $1$ with probability $\frac{1}{n} + O(1/n^2)$. We can then bound,
    \begin{align*}
         &\Ex_{\bi \sim \Unif([n])} \bracket*{\KL{\mcM_{\phi}(S)}{\mcM_{\phi}'(S_{-{\bi}})}} \\
         &\quad\quad= \frac{1}{n}\cdot  \KL{\Ber(1/n)}{\Ber(1/n^2)} + (1-1/n)\cdot \KL{\Ber(1/n)}{\Ber(1/n + O(1/n^2))} \\
         &\quad\quad\leq \frac{1}{n} \cdot \tilde{O}(1/n) + (1 - 1/n) \cdot O(1/n^2) = \tilde{O}(1/n^2).
    \end{align*}
\end{enumerate}

This second case exemplifies why ALOOKL stability works better than differential privacy in this setting. If $S$ contains exactly one point on which $\phi$ evaluates to one, then when $S_{-\bi}$ happens to remove that special point, the cost is quite high. However, this cost is modulated by the fact that the special point is only removed with probability $1/n$. In contrast, since DP chooses $S'$ to be a worst-case neighbor, this high cost is not modulated.

\pparagraph{ALOOKL stability bounds bias.} Let $\mcM$ be a ALOOKL stable randomized algorithm that takes as input a sample set $S \in X^n$ and outputs a test function $\phi:X^1 \to [0,1]$. Consider the following three experiments.
\begin{enumerate}
    \item In the first experiment, we draw $\bS \sim \mcD^n$, $\bphi \sim \mcM(\bS)$, and $\bi \sim [n]$. Then, we output $\bphi(\bS_{\bi})$. 
    \item In the second experiment, we similarly draw $\bS \sim \mcD^n$ and $\bphi \sim \mcM(\bS)$. This time, we draw $\bx \sim \mcD$ we output $\bphi(\bx)$.
    \item In the third experiment, we similarly draw $\bS \sim \mcD^n$ and $\bi \sim [n]$. This time, we draw $\bphi' \sim \mcM'(\bS_{-\bi})$ and output $\bphi'(\bS_{\bi})$. Here, $\mcM'$ is the randomized algorithm guaranteed to exist by \Cref{def:ALKL-intro}.
\end{enumerate}
The goal is to show that the first and second experiments have similar expectations. We do this by using the third experiment as a bridge. The first and third experiments are close by the guarantee of ALOOKL stability -- in particular, even conditioned on the choice of $\bS$ and $\bi$ (which fixes $\bS_{\bi}$), the distributions of $\bphi$ and $\bphi'$ will be close. 

The second and third experiments are close because, in both cases, the input to the test ($\bx$ for the second experiment and $\bS_{\bi}$ for the third experiment) is distributed according to $\mcD$ and independent of the test used ($\bphi$ in the second experiment and $\bphi'$ in the third). For the second experiment, this independence is immediate from the experiment setup. For the third experiment, it holds because the drawing of $\bphi'$ only depends on $\bS_{-\bi}$ which is independent of $\bS_{\bi}$. Therefore, the only difference between the second and third experiments is in the marginal distributions of $\bphi$ and $\bphi'$, and this difference is once again guaranteed to be small by the ALOOKL stability condition.

\subsection{The remainder of this technical overview}
\label{subsec:remainder-overview}
Generalizing the ideas from \Cref{subsec:ALKL-overview} to subsampling queries that take as input more than one point requires a fair bit of work. Instead, we give a technically simpler and quantitatively stronger proof that relies on a natural generalization of ALOOKL stability, which we call \emph{average leave-many-out KL stability} (ALMOKL stability). We describe this new form of stability and analyze the ALMOKL stability of subsampling queries in \Cref{subsec:ALMOKL-overview}. This analysis involves a new generalization of Hoeffding's celebrated reduction theorem \cite{Hoe63}, described in \Cref{subsec:concentration-overview}, which may be of independent interest.

Rather than directly bound bias using ALMOKL stability, we use mutual information as a technical intermediate. Specifically, in \Cref{subsec:mutual-info-overview}, we describe a proof of the following. If the analyst adaptively asks a series of adaptively chosen queries to a sample $\bS \sim \mcD^n$, then the mutual information between $\bS$ and all the query responses is bounded by the sum of the ALMOKL stabilities of the queries. Mutual information is well-known to bound bias in a variety of contexts. In \Cref{sec:tech-overview-bias}, we use that well-known connection to give formal statements of our bias bounds (formalizing \Cref{thm:informal}).

As is typical of generalization results based on mutual information, our guarantees in \Cref{sec:tech-overview-bias} only bound the \emph{expected} bias. While one can use Markov's inequality to get an explicit dependence on the failure probability, that dependence will be polynomial in the inverse failure probability. In \Cref{subsec:autoboost-overview}, we describe how to improve that dependence (in some settings) to logarithmic in the inverse failure probability. This involves a generalization of a direct product theorem of Shaltiel \cite{Sha04} which may also be of independent interest.

\subsection{The average leave-many-out KL stability of subsampling queries}
\label{subsec:ALMOKL-overview}
We introduce the following notion of \emph{algorithmic stability}. 
\begin{definition}[Average leave-many-out KL stability]
    \label{def:ALMOKL-intro}
    A randomized algorithm $\mcM: X^n \to Y$ is \emph{ALMOKL stable} with parameters $(m, \eps)$ if there is a randomized algorithm $\mcM': X^{n-m} \to Y$ such that, for all samples $S \in X^n$,
    \begin{equation*}
        \Ex_{\bJ \sim \binom{[n]}{n-m}} \bracket*{\KL{\mcM(S)}{\mcM'(S_{\bJ})}} \leq \eps
    \end{equation*}
    where $\KL{\cdot}{\cdot}$ is the KL-divergence, $\bJ$ is drawn uniformly from all size-$(n-m)$ subsets of $[n]$, and $S_J$ indicates the size-$(n-m)$ subset of $S$ specified by the indices within $J$.
\end{definition}
This definition generalizes ALOOKL (\Cref{def:ALKL-intro}) stability, which corresponds to ALMOKL stability with $m = 1$. At first glance, this more general definition seems harder to satisfy because $\mcM'$ must predict the behavior of $\mcM(S)$ when a large fraction of points from $S$ are removed, rather than when just a single point is removed. However, our analysis also allows for a larger KL divergence in the case where $m$ is large -- specifically, we will later see the goal is to minimize the ratio $\eps/m$. Therefore, by setting $m$ larger, we can afford a weaker bound on $\eps$. Indeed, for subsampling queries, we find using $m \approx n/2$ gives a stronger result \emph{and} simpler analysis than $m=1$.

To use ALMOKL stability in our analysis, we must bound the stability of subsampling queries.
\begin{lemma}[Subsampling queries are ALMOKL stable]
    \label{lem:ALMOKL-stab-intro}
    For any function $\phi:X^w \to Y$ and $m \leq n$, the subsampling query $\phi$ is $(m, \eps)$-ALMOKL stable for
    \begin{equation*}
        \eps \coloneqq \frac{w(|Y|-1)}{n-m+1}.
    \end{equation*}
\end{lemma}
As discussed earlier, the ratio $\eps/m$ will play a crucial role in our analysis. By setting $m = \ceil{n/2}$, this ratio becomes $O(w |Y|/n^2)$, which is enough to prove our results.

\subsubsection{Bounding the stability of simple subsampling queries} 
\label{subsubsec:sketch-w=1}
We sketch a proof of \Cref{lem:ALMOKL-stab-intro} in the simplest setting where $w = 1$, $Y = \zo$, and $m = n/2$. Let $\mcM_{\phi}:X^n \to \zo$ be the randomized algorithm mapping $S$ to $\phi(S)$. Then, the randomized algorithm $\mcM_{\phi}':X^{n/2} \to \zo$ will do the following: With probability $\frac{1}{n}$, it outputs a uniform bit in $\zo$. Otherwise, it runs the subsampling query $\phi$ on its input (i.e. outputs $\phi(\bx)$ for $\bx$ chosen uniformly from its input). Note this is essentially the same construction of $\mcM_{\phi}'$ used in \Cref{subsec:ALKL-overview}, except in this case, we output a random bit with probability $\frac{1}{n}$ rather than $\frac{1}{n^2}$.

We aim to show, for any $S \in X^n$, that
\begin{equation*}
     \Ex_{\bJ \sim \binom{[n]}{n-m}} \bracket*{\KL{\mcM(S)}{\mcM'(S_{\bJ})}} \leq O\paren*{\lfrac{1}{n}}.
\end{equation*}
Consider two cases:
\begin{enumerate}
    \item If $\phi$ evaluates to $0$ on \emph{every} $x \in S$, then $\mcM(S) \sim \Ber(0)$ and regardless of what $J$ is chosen, $\mcM'(S_{J}) \sim \Ber(1/2n)$. In this case,
    \begin{equation*}
        \Ex_{\bJ \sim \binom{[n]}{n-m}} \bracket*{\KL{\mcM(S)}{\mcM'(S_{\bJ})}} = \KL{\Ber(0)}{\Ber(1/2n)} = O\paren*{\lfrac{1}{n}}.
    \end{equation*}
    A symmetric argument holds if $\phi$ evaluates to $1$ on every $x \in S$.
    \item Otherwise, let $p$ be the fraction on $x \in S$ on which $\phi$ evaluates to $1$ (in this case, $p \in [1/n, 1-1/n]$). Then, for $\bJ\sim \binom{[n]}{n-m}$, let $\bq$ be the probability that $\mcM'(S_{\bJ})$ evaluates to $1$. We observe that $\abs*{\Ex[\bq] - p} \leq 1/n$ and that $\Var[\bq] = O(p(1-p)/n)$. Using a second-order approximation of the KL divergence,
    \begin{align*}
        \Ex\bracket*{\KL{\Ber(p)}{\Ber(\bq)}} &\approx \Ex\bracket*{\frac{(p - \bq)^2}{2p(1-p)}} \\
        &= \frac{(\Ex[\bq] - p)^2}{2p(1-p)} + \frac{\Var[\bq]}{2p(1-p)}\\
        &\leq \frac{1/n^2}{2p(1-p)} + \frac{O(\frac{p(1-p)}{n})}{2p(1-p)} = O\paren*{\lfrac{1}{n}}.
    \end{align*}
\end{enumerate}

\subsubsection{Handling larger subsampling queries using the concentration of U-statistics}
\label{subsec:concentration-overview}
To prove \Cref{lem:ALMOKL-stab-intro}, we need to generalize the sketch in \Cref{subsubsec:sketch-w=1} to handle arbitrary ranges $Y$, arbitrary $m$, and arbitrary subsample size $w$. It turns out that generalizing to arbitrary $Y$ and $m$ is straightforward, so in this overview, we will only discuss how to handle larger subsamples. We generalize to larger $w$ by proving the following: 
\begin{quote}
    \emph{Roughly speaking, the stability of queries using subsamples of size-$w$ on size-$n$ datasets is as good as queries using subsamples of size-$1$ on size-$\floor{\frac{n}{w}}$ datasets.} \hfill{($\diamondsuit$)}
\end{quote}
Since the ALMOKL stability of subsampling queries with $w = 1$ is $\approx 1/n$, this implies that for larger $w$, that stability is $\approx w/n$.\footnote{Note, that if we use ALOOKL stability, as in \Cref{subsec:ALKL-overview}, we would desire the stability for $w = 1$ to be $\approx 1/n^2$. Then, an attempt to use the same strategy to generalize to arbitrary $w$ would give a stability of $\approx 1/(\floor{n/w})^2 \approx w^2/n^2$. Here, the dependence on $w$ is quadratic rather than linear, which is suboptimal.}

To show the statement in $(\diamondsuit)$ holds, we generalize a classic inequality by Hoeffding.

\begin{theorem}[Hoeffding's reduction theorem \cite{Hoe63}]
    \label{thm:hoef-63-intro}
    Consider any finite population $X \in \R^n$ and integer $m \leq n$. Let $\bx_1, \ldots, \bx_m$ be uniform samples from $X$ chosen \emph{without} replacement, and $\by_1, \ldots, \by_m$ be uniform samples from $X$ chosen \emph{with} replacement. Then, for any continuous and convex $f: \R \to \R$,
    \begin{equation*}
        \Ex\bracket*{f\paren*{\frac{1}{m} \sum_{i \in [m]} \bx_i}} \leq \Ex\bracket*{f\paren*{\frac{1}{m} \sum_{i \in [m]} \by_i}}.
    \end{equation*}
\end{theorem}
\Cref{thm:hoef-63-intro} give a simple approach for proving concentration inequalities on the sum of random variables sampled without replacement. Since $x \mapsto e^{\lambda x}$ is convex, it implies that the moment generating function of $\sum_i \bx_i$ is upper bounded by that of $\sum_i \by_i$. Therefore, Chernoff bounds that apply to $\sum_i \by_i$ also apply to $\sum_i \bx_i$.

Our generalize of Hoeffding's reduction theorem uses a few pieces of notation: For any $S \in X^n$ and $w \leq n$, let $\bT \sim S^{(w)}$ indicate that $\bT$ contains $w$ element from $S$ sampled uniformly without replacement. Furthermore, for any $\phi:X^w \to \R$, let $\mu_{\phi}(S) \coloneqq \Ex_{\bT \sim S^{(w)}}[\phi(\bT)]$.

\begin{theorem}[Extension of Hoeffding's reduction theorem to U-statistics]
    \label{thm:u-stat-convex}
    Consider any integers $w \leq m \leq n$, finite population $S \in X^n$, and $\phi:X^w \to \R$. For any convex $f: \R \to \R$,
     \begin{equation*}
         \Ex_{\bJ \sim \binom{[n]}{m}}\bracket*{f\paren*{\mu_{\phi}(S_{\bJ})}} \leq \Ex_{\bT_1, \ldots, \bT_k \iid S^{(w)}}\bracket*{f\paren*{\frac{1}{k}\sum_{i \in [k]}\phi(\bT_i)}}
     \end{equation*}
     where $k = \floor{m/w}$.
\end{theorem}
In the statistics literature, the random variable $\mu_{\phi}(S_{\bJ})$ is called the \emph{U-statistic} of order m with kernel $\phi(\cdot)$. This quantity was introduced in \cite{Hoe48} and since then has received extensive study (see, for example, the text book \cite{KB13}). Despite that extensive study, to the best of our knowledge, no inequality of the form given in \Cref{thm:u-stat-convex} is known. The closest we are aware of is the recent work \cite{AKW22} which showed that $\mu_{\phi}(S_{\bJ})$ satisfies appropriate variants of Hoeffding's and Bernstein's concentration inequalities. Such concentration inequalities are easily corollaries of \Cref{thm:u-stat-convex}, which can be used to give a suitable bound on the moment generating function of $\mu_{\phi}(S_{\bJ})$. Indeed, our proof of \Cref{thm:u-stat-convex} borrows ideas from and simplifies \cite{AKW22}'s proof.

We use \Cref{thm:u-stat-convex} to show that $(\diamondsuit)$ holds. Given the ubiquity of U-statistics, it may also be of independent interest.

\subsection{Stability and mutual information}
\label{subsec:mutual-info-overview}
Next, we show how to use ALMOKL stability to bound the the mutual information between a sample and an adaptive series of queries. Combining the following with the stability bound in \Cref{lem:ALMOKL-stab-intro} recovers \Cref{thm:MI-intro}.
\begin{theorem}[Using ALMOKL stability to upper bound mutual information]
    \label{thm:stab-to-MI-intro}
    Let $\bS$ drawn from a product distribution over $X^n$ and, for each $t \in [T]$, draw $\by_t \sim \mcM^{\by_1, \ldots, \by_{t-1}}(\bS)$. Then, for any $m \in [n]$,
    \begin{equation*}
        I(\bS; (\by, \ldots, \by_T)) \leq \frac{n}{m} \cdot \Ex_{\by}\bracket*{\sum_{t \in T} \stab_m\paren*{\mcM^{\by_1, \ldots ,\by_{t-1}}}},
    \end{equation*}
    where $\stab_m(\mcM)$ refers to the infimum over $\eps$ such that $\mcM$ is $(m,\eps)$-ALMOKL stable
\end{theorem}

Every time a new query response is revealed, $\by_t \sim \phi_t(\bS)$, the distribution of $\bS$ conditioned on that response shifts. Our analysis will track a ``potential" of that distribution. The most direct quantity to track would be the entropy, as the mutual information can be written as
\begin{equation}
    \label{eq:MI-tech-overview}
    I(\bS; \by) \coloneqq H(\bS) - H(\bS \mid \by) = H(\bS) - \Ex_{\by' \sim \mcD_y}\bracket*{H(\bS \mid \by = \by')}
\end{equation}
where $\mcD_y$ is the marginal distribution of $\by$. Therefore, if we could show that \emph{regardless} of the distribution of $\bS$ conditioned on the prior responses $\by_1, \ldots, \by_{t-1}$, that the expected drop in entropy of $\bS$ by learning $\by_t$ is at most $\tfrac{n}{m} \cdot \stab_m(\mcM_t)$, then \Cref{thm:stab-to-MI-intro} holds. Such a bound holds, but only when $\bS$ comes from a product distribution.
\begin{fact}[Mutual information when the sample comes from a product distribution]
    \label{fact:product-MI}
    For any randomized algorithm $\mcM$, $\bS$ sampled from a product distribution on $X^n$, and $m \in [n]$,
    \begin{equation*}
        I(\bS; \mcM(\bS)) \leq \frac{n}{m} \cdot \stab_m(\mcM).
    \end{equation*}
\end{fact}
Recall that, at the start, $\bS$ is sampled from $\mcD^n$ which \emph{is} a product distribution. Unfortunately, as we receive query responses, the distribution of $\bS$ conditioned on those responses shifts and can become non-product, and \Cref{fact:product-MI} is \emph{false} without the assumption that $\bS$ is product.
\begin{example}[\Cref{fact:product-MI} is false for non-product distributions]
    \label{example:non-product-MI} Let $X = \zo$ and $\bS \in X^n$ be either the all $1$s vector or all $0$s vector, each with equal probability. For $\phi:X^1 \to \zo$ being the identity function and $\mcM_{\phi}$ be the randomized algorithm corresponding to this subsampling query. Then, by \Cref{lem:ALMOKL-stab-intro}, $\mcM_{\phi}$ is $(m=n/2,\eps=O(1/n))$-ALMOKL stable. However, for $\by \sim \mcM_{\phi}(\bS)$, we see that\footnote{We compute entropy using the natural base.} $H(\bS) = \ln 2$, but after conditioning on either $\by = 1$ or $\by = 0$, $\bS$ is fully determined and so has entropy $0$. Therefore, $I(\bS,\by) = \ln 2$, a factor of $n$ larger than the bound in \Cref{fact:product-MI}.
\end{example}
A key observation about \Cref{example:non-product-MI} is that while $\by$ reveals much information about $\bS$, it also makes $\bS$ ``more product": Before conditioning on $\by$, $\bS$ has a non-product distribution, but after conditioning on any choice for $\by$, it becomes a product distribution. In light of this observation, we could hope that there is a trade-off between how a query response affects the entropy of $\bS$ and the ``productness" of $\bS$. This suggests an \emph{amortized} analysis. While query responses that drastically decrease the entropy of $\bS$ can happen, perhaps they need to be proceeded by many queries that decrease the ``productness" of $\bS$ and are therefore rare.

Indeed, we use such an amortized analysis. For the purpose of this exposition, assume $n$ is even and fix $m = n/2$. For this choice of $m$, we track the ``half-conditional entropy" of $\bS$. 
\begin{equation}
    \label{eq:def-half-conditional-entropy}
    H_{\mathrm{half}}(\bS) \coloneqq \Ex_{\bJ \sim \binom{[n]}{n/2}}\bracket*{H(\bS_{\bJ} \mid \bS_{-\bJ})} \overset{\text{\Cref{prop:half-conditional}}}{=} \frac{H(\bS) - \Ex_{\bJ}\bracket*{I(\bS_{\bJ} ; \bS_{-\bJ})}}{2}
\end{equation}
where $\bJ$ is a uniform size-$(n/2)$ subset of $[n]$, and $S_J$ and $S_{-J}$ are the two disjoint subsets of $S$ indexed by $J$ and its compliment respectively. The term $I(\bS_{\bJ} ; \bS_{-\bJ})$ measures how much knowing half the elements of $\bS$ affects the distribution of the other half of elements and is natural measure of ``productness." Fortunately, we find that learning the output of a ALMOKL stable randomized algorithm cannot change the half-conditional entropy of the sample by much.
\begin{lemma}[Bounding expected drop in half-conditional entropy]
    \label{lem:half-conditional-drop}
    For any randomized algorithm $\mcM:X^n \to Y$ and random variable $\bS$ supported on $X^n$ (including non-product distributions),
    \begin{equation*}
        H_{\mathrm{half}}(\bS) - H_{\mathrm{half}}(\bS \mid \mcM(\bS)) \leq \stab_{n/2}(\mcM).
    \end{equation*}
\end{lemma}
In addition to \Cref{lem:half-conditional-drop}, we need one last ingredient.
\begin{lemma}[Half-conditional entropy bounds mutual information]
    \label{lem:half-conditional-to-mutual}
    For any random variables $\bS, \by$ where the distribution of $\bS$ is product on $X^n$,
    \begin{equation*}
        I(\bS; \by) \leq 2\cdot \paren*{ H_{\mathrm{half}}(\bS) - H_{\mathrm{half}}(\bS \mid \by)}.
    \end{equation*}
\end{lemma}
Given these ingredients, the proof of \Cref{thm:stab-to-MI-intro} in the special case of $m = n/2$ is simple. By repeatedly applying \Cref{lem:half-conditional-drop} to each query response,
\begin{equation*}
    H_{\mathrm{half}}(\bS) - H_{\mathrm{half}}(\bS \mid (\by, \ldots, \by_T)) \leq \Ex_{\by}\bracket*{\sum_{t \in T} \stab_{n/2}\paren*{\mcM^{\by_1, \ldots ,\by_{t-1}}}}.
\end{equation*}
The desired bound then follows from \Cref{lem:half-conditional-to-mutual}. We note that this application of \Cref{lem:half-conditional-to-mutual} does require the \emph{starting} distribution of $\bS$ to be product, but we can apply it even if the intermediate distributions of $\bS$ conditioned on some of the responses is non-product. This is because \Cref{lem:half-conditional-drop} does \emph{not} require productness, so we can bound the drop in $m$-conditional entropy of the entire sequence of adaptively chosen queries.

In \Cref{sec:stability-to-MI}, we generalize these ideas to other choices of $m$. Feldman and Steinke also show how to bound the mutual information using ALOOKL stability, corresponding to \Cref{thm:stab-to-MI-intro} in the case of $m=1$ \cite{FS18}. Our techniques are similar to theirs, appropriately generalized for $m > 1$. Our main contributions to this notion of stability is the idea to generalize the definition to the $m > 1$ case as well as a unique presentation, particularly the introduction of half-conditional entropy.

\subsection{Formal statements of our bias bounds}
\label{sec:tech-overview-bias}
Mutual information is known to bound bias in a variety of contexts (see e.g. \cite{RZ16,FS18}). Similarly, we use the mutual information bound of \Cref{thm:MI-intro} to bound the bias of responses to subsampling queries. The purpose of this subsection will be to formalize and explain our bias bounds, and we'll defer the (mostly standard) proofs to \Cref{sec:gen}. We begin by formalizing \Cref{thm:informal}.

\begin{theorem}[Formal version of \Cref{thm:informal}]
    \label{thm:formal-simple}
    For any distribution $\mcD$ over domain $X$ and analyst making a series of adaptive queries $\bphi_1, \ldots, \bphi_T:X^w \to Y$ to a sample $\bS \sim \mcD^n$,
    \begin{equation}
        \label{eq:bias-formal-simple}
        \Ex_{\bS, \bphi_1, \ldots, \bphi_T}\bracket*{\sup_{t \in [T], y \in Y}\paren*{\Prx_{\bT \sim \bS^{(w)}}\bracket*{\bphi_t(\bT) = y} - \Prx_{\bT \sim \mcD^w}\bracket*{\bphi_t(\bT) = y}}^2} \leq  O\paren*{\frac{w^2T |Y|}{n^2} + \frac{w\log (T|Y|)}{n}}
    \end{equation}
\end{theorem}

The sample size in \Cref{thm:informal} is set so that the right-hand side of \Cref{eq:bias-formal-simple} is an arbitrarily small constant. More generally, by Markov's inequality, with probability at least $0.9$ over $\bS$ and the analyst's choices of queries, for all $t \in [T]$ and $y \in Y$
\begin{equation*}
     \abs*{\Prx_{\bT \sim \bS^{(w)}}\bracket*{\bphi_t(\bT) = y} - \Prx_{\bT \sim \mcD^w}\bracket*{\bphi_t(\bT) = y}} \leq O\paren*{\frac{w\sqrt{T|Y|}}{n} + \sqrt{\frac{w\log(T|Y|)}{n}}}.
\end{equation*}
We interpret each of the two terms in the above equation separately.
\begin{enumerate}
    \item The second term, $\tilde{O}(\sqrt{\frac{w}{n}})$, is simply the \emph{inherent} bias in a sample: There are\footnote{One such example, for $\mcD = \Unif(\{-1,1\})$, is the query $\phi(x_1, \ldots, x_w) = \Ind[x_1 + \cdots + x_w\geq 0]$} queries $\phi:X^w \to \zo$, for which,
    \begin{equation*}
        \Ex_{\bS \sim \mcD^n}\bracket*{\abs*{\Prx_{\bT \sim \bS^{(w)}}\bracket*{\phi(\bT) = 1} - \Prx_{\bT \sim \mcD^w}\bracket*{\phi(\bT) = 1}}} = \Theta\paren*{\sqrt{\lfrac{w}{n}}}.
    \end{equation*}
    Observe that, in the above, the query $\phi$ is fixed and therefore independent of the sample $\bS$.
    \item The first term therefore quantifies the extra bias incurred due to the adaptively selected queries.  When $T |Y| \leq \frac{n}{w}$, that first term is dominated by the inherent bias. The guarantee then slowly degrades from the inherent bias to vacuous as $T|Y|$ varies between $\frac{n}{w}$ and $\frac{n^2}{w^2}$. In contrast, a naive approach that uses disjoint samples for each query works only for $T \leq \frac{n}{w}$.
\end{enumerate}

Our next theorem generalizes \Cref{thm:formal-simple} in several ways. First, we disconnect the test queries from the first $T$ queries the analyst asks. After receiving the response $\by_1, \ldots, \by_T$, the analyst chooses any set of test queries $\psi_1:X^{v_1} \to [0,1], \ldots, \psi_m:X^{v_m} \to [0,1]$ for which we bound $\abs{\Ex_{\by \sim \psi(\bS)}[\by] - \Ex_{\by \sim \psi(\mcD)}[\by]}$. To recover \Cref{thm:formal-simple}, we define a test query $\psi_{t,y}(x_1, \ldots, x_w) \coloneqq \Ind[\psi_t(x_1, \ldots, x_w)=y]$ for each $t \in [T]$ and $y \in Y_t$. The following notation will be convenient.

\begin{definition}[Expectation and variance of a query]
    \label{def:expectation}
    For a query $\phi:X^w \to Y \subseteq \R$ and sample $S \in X^n$, we use the notation $\mu_{\phi}(S)$ as shorthand for $\Ex_{\by \sim \phi(S)}[\by]$ and $\Var_{\phi}(S)$ as shorthand for $\Varx_{\by \sim \psi(S)}[\by]$. Similarly, for a distribution $\mcD$ over domain $X$, we define $\mu_{\phi}(\mcD) \coloneqq \Ex_{\by \sim \phi(\mcD)}[\by]$ and $\Var_{\phi}(\mcD) \coloneqq \Varx_{\by \sim \psi(\mcD)}[\by]$
\end{definition}

Intuitively, when $\psi(\mcD)$ concentrates tightly, it is easier to give a good estimate of $\mu_{\psi}(\mcD)$. The second improvement of \Cref{thm:main-binary} is that it gives a \emph{variance-dependent} error guarantee that improves when $\Var_{\psi}(\mcD)$ is small. In contrast, \Cref{thm:formal-simple} uses the pessimistic bound that $\Var_{\psi}(\mcD) \leq 1$.

\begin{definition}[Error of a query]
    \label{def:error-simple}
    For any $\psi:X^w \to [0,1]$, distribution $\mcD$ over $X$ and sample $S \in X^n$, we define
    \begin{equation*}
        \error(\psi, S, \mcD) \coloneqq \frac{1}{w} \cdot \min\paren*{\Delta, \frac{\Delta^2}{\Var_{\psi}(\mcD)}} \quad\quad\text{where }\Delta \coloneqq \abs{\mu_{\psi}(S) - \mu_{\psi}(\mcD)}. 
    \end{equation*}
\end{definition}
If $\error(\psi, S, \mcD)\leq \eps$, then $\Delta \leq \max(w\eps, \sqrt{w\eps \Var_{\psi}(\mcD)})$. When $\Var_{\psi}(\mcD) = 1$, the second term or the trivial bound of $\Delta \leq 1$ dominates, but a lower $\Var_{\mcD}(\psi)$ can lead to a substantially better bound on $\Delta$.

Lastly, we allow the analyst to choose a different domain and range size for each query -- As a function of the responses $\by_1, \ldots, \by_{t-1}$, the analyst chooses $w_t, Y_t$ and a subsampling query $\phi:X^{w_t} \to Y_t$. Our bound will depend on the expected total ``cost" of queries the analyst asks, where the ``cost" of a query $\phi:X^w \to Y$ is $w|Y|$. 

\begin{theorem}[Generalization of \Cref{thm:formal-simple}]
    \label{thm:main-binary}
     For any distribution $\mcD$ over domain $X$ and analyst that makes a series of adaptive queries $\bphi_1:X^{\bw_1} \to \bY_1, \ldots, \bphi_T:X^{\bw_T} \to \bY_T$ to a sample $\bS \sim \mcD^n$ and then chooses a collection of tests $|\boldsymbol{\Psi}| \leq m$,
    \begin{equation*}
        \Ex\bracket*{\sup_{\psi \in \boldsymbol{\Psi}} \error(\psi, \bS, \mcD)} \leq O \paren*{\frac{\Ex\bracket*{\sum_{t \in T}\bw_t |\bY_t|}}{n^2} + \frac{\log m + 1}{n}}
    \end{equation*}
    where the expectation is both over the sample $\bS \sim \mcD^n$ and the analyst's decisions.
\end{theorem}

\subsection{A self-reduction for boosting success probability}
\label{subsec:autoboost-overview}

As is typical of mutual information-based bounds, \Cref{thm:main-binary} only guarantees low bias in expectation. When $w_t = 1$ for all $t \in [T]$, we further show a guarantee that holds with an exponentially small failure probability.

\begin{theorem}[Improved dependence on failure probability]
    \label{thm:high-probability}
    In the setting of \Cref{thm:main-binary}, if the analyst's choices of queries satisfy $\bw_t = 1$ for all $t \in [T]$ and $\sum_{t \in [T]}|\bY_t| \leq b$ almost surely, then, for any failure probability $\delta > 0$,
    \begin{equation*}
        \Pr\bracket*{\sup_{\psi \in \boldsymbol{\Psi}} \error(\psi, \bS, \mcD)\geq O\paren*{\ln(m/\delta)\cdot\paren*{\frac{b}{n^2} +\frac{1}{n}}}} \leq \delta.
    \end{equation*}
\end{theorem}
The case of $w = 1$ is particularly important for two reasons.
\begin{enumerate}
    \item It is sufficient for our application of answering statistical queries, a widely-applicable query class, given in \Cref{subsec:SQ}. Indeed, one way to characterize statistical queries are precisely those queries $\phi:X^n \to [0,1]$ for which an unbiased and bounded estimator of $\mu_{\phi}(S)$ can be computed given a single $\bx \sim \Unif(S)$. Our mechanism for answering statistical queries simply averages many of these unbiased estimators.
    \item One way to answer a query with $w \geq 2$ is to cast a sample of $n$ points from $\mcD$ as a sample of $\floor{n/w}$ points each drawn from $\mcD^w$. By doing so, each query $\phi:X^w \to Y$ can be answered by looking at one ``grouped point," and so \Cref{thm:high-probability} gives a high probability guarantee. We conjecture that such grouping is unnecessary and that a variant of \Cref{thm:high-probability} would directly hold without the restriction that $w_t = 1$ for all $t \in [T]$. That said, the proof breaks in this setting, and so we consider extending it to be an intriguing open problem.
\end{enumerate}

The high probability guarantee of \Cref{thm:high-probability} does not follow from only a mutual information bound. Instead, we use mutual information to achieve a constant failure probability and separately prove a reduction from the low failure probability case to that of constant failure probability.

\begin{lemma}[Informal version of \Cref{lem:auto-boost}]
    \label{lem:auto-boost-informal}
    In the context of \Cref{thm:main-binary}, suppose that the analyst always asks subsampling queries $\phi:X^{w_t} \to Y_t$ where $w_t = 1$ for all $t \in [T]$. Any low-bias guarantee that holds with constant probability for a sample size of $n$ holds with probability $1 - 2^{\Omega(-N/n)}$ given a sample of size $N > n$.
\end{lemma}
The starting point of \Cref{lem:auto-boost-informal} is the following observation. Given a sample of size $S \in X^{N \coloneqq nk}$, the answer of a query $\by = \phi(S)$ where $\phi:X^1 \to Y$ is equivalent to the answer $\by' = \phi(S_{\bi})$ where $S_1, \ldots S_k \in X^n$ are equal-sized partitions of $S$ and $\bi \sim \Unif([k])$, as depicted in \Cref{fig:reduction}. 

\begin{figure}[htbp]
  \centering
  \includegraphics{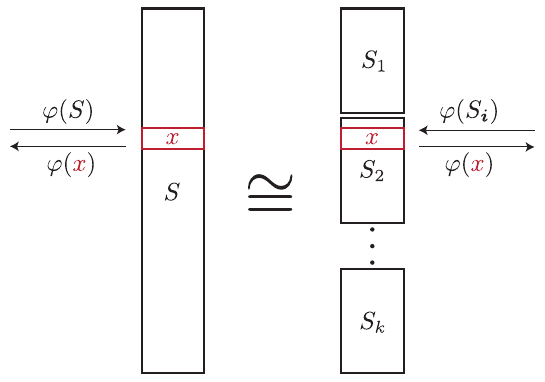}
  \caption{An illustration of why, for $\phi:X^1 \to Y$, the distribution of $\phi(S)$ is equivalent to $\phi(S_{\bi})$ where $S_1, \ldots, S_k$ are partitions of $S$ and $\bi \sim \mathrm{Unif}([k])$.}
    \label{fig:reduction}
\end{figure}

The hypothesis of \Cref{lem:auto-boost-informal} says that, if we look at a single $i \in [k]$ in isolation, the probability that $\bS_i$ is biased is at most a constant, say $0.1$. Therefore, if these events were independent, the probability a large fraction, say $0.2k$, of the disjoint subsamples are biased is $2^{-\Omega(k)}$.

However, because the analyst can choose which query to ask a group $\bS_i$ as a function of responses from some other group $\bS_{i'}$, the events indicating whether each group is biased need not be independent. To handle this we generalize a direct product theorem of Shaltiel that was originally applied to fair decision trees \cite{Sha04}. That generalization, given in \Cref{lem:direct-product}, shows that while those events need not be independent, they behave no worse than if they were.

Using the above direct product theorem, we determine that the probability a constant fraction of $\bS_1, \ldots, \bS_k$ are biased is $2^{-\Omega(k)}$. Finally, we connect this to the bias of $\bS$: Since each $\bS_i$ is a subsample of $\bS$, if $\bS$ were biased, it is likely that many of $\bS_1, \ldots, \bS_k$ are also biased. This allows us to conclude that the probability $\bS$ is biased is not much larger than the probability many of $\bS_1, \ldots, \bS_k$ are also biased.

\section{Improvements upon the conference version}
\label{sec:comparison}

In the conference version of the paper \cite{B23STOC}, \Cref{thm:MI-intro} was weaker by a multiplicative $\log n$ factor, which similarly affects \Cref{thm:informal,thm:formal-simple,thm:main-binary}. As a result, the statistical query mechanism described in \Cref{fig:SQ-mechanism} was not known to be state of the art. To get around this, the conference version showed that a small amount of noise could be added to each query response to remove this $\log n$ factor: Specifically, we replace each subsampling query $\phi:X^w \to Y$ with $\phi':X^w \to Y$ where $\phi'(x_1, \ldots, x_w)$ outputs $\Unif(Y)$ with a small, carefully chosen probability, and otherwise $\phi(x_1, \ldots, x_w)$. It was therefore able to recover the quantitative parameters of \Cref{thm:SQ} but with a more complicated mechanism.

To remove this multiplicative $\log n$ factor, we revamped many of the proofs. Most notably, this involved developing a new notion of stability. The conference version used average leave-\emph{one}-out KL stability (\Cref{def:ALKL-intro}), which this version generalizes to average leave-\emph{many}-out KL stability (\Cref{def:ALMOKL-intro}). In \Cref{sec:tightness-bad}, we show that a bound using Feldman and Steinke's framework directly necessarily has this additional $\log n$ factor. Interestingly, in addition to giving a better quantitative result, this new version of stability leads to simpler and shorter proofs.

Beyond strengthening our result, we hope the new ingredients developed in this version will be independent interest. This stability notion, and our new generalization of Hoeffding's reduction theorem \Cref{thm:u-stat-convex} may have other applications. Furthermore, much of the presentation, particularly \Cref{sec:technical-overview}, has been revamped to provide more intuition.
\section{Notation and key definitions}
\label{sec:prelim}

\paragraph{Indexing a sample}
For a natural number $n$, we use $[n]$ to denote the set $\set{1,\ldots, n}$. For a tuple $S \in X^n$ we use the notation $S_i$ to indicate the $i^\text{th}$ element of $S$, and $S_{-i} \coloneqq (S_1, \ldots, S_{i-1}, S_{i+1}, \ldots, S_n)$ denotes the remaining $n-1$ elements. For $w \leq n$, we use the notation $\binom{S}{w}$ to indicate the set of all $\binom{n}{w}$ unordered size-$w$ multisets $S'$ that are contained within $S$. For any set of indices $J \in \binom{S}{w}$, we use $S_{J} \coloneqq (S_{J_1}, \ldots, S_{J_w})$ to denote the elements of $S$ indexed by $J$, and $S_{-J} \coloneqq S_{\overline{J}}$ for the elements indexed by its compliment. We further use $S^{(w)}$ to indicate the multiset of all $n(n-1)\cdots(n-w+1)$ \emph{ordered} size-$w$ tuples contained within $S$, i.e. all $(x_1, \ldots, x_w)$ such that $x_1 \in S$, $x_2$ is in $S$ with $x_1$ removed, and so on.

\paragraph{Random variables and distributions.}
We use \textbf{boldfont} to denote random variables and calligraphic font to denote distributions (e.g. $\bx \sim \mcD$). The notation $\bx \sim S$ for a (multi)set or tuple $S$ is shorthand for a uniform sample from $S$, $\bx \sim \Unif(S)$. For a distribution $\mcD$, we will use $\bx_1, \ldots, \bx_w \iid \mcD$ or equivalently $\bx \sim \mcD^w$ to denote that $\bx_1, \ldots, \bx_w$ are independent and distributed according to $\mcD$. For example, $\bx_1, \ldots, \bx_w \iid S$ are $w$ samples chosen uniformly \emph{with} replacement from $S$, whereas $\by_1, \ldots, \by_w \sim S^{(w)}$ are $w$ samples chosen uniformly \emph{without} replacement. For a distribution $\mcD$ over domain $X$ and element $x \in X$, we use $\mcD(x)$ to denote the probability mass function of $\mcD$ evaluated at $x$. For notational convenience, all domains and distributions will be discrete, though our results could be extended to the setting where distributions are continuous.\footnote{In more detail, by keeping all distribution discrete, we can directly reason about drops in entropy. This eases notation and, we hope, gives more intuition to the reader. In the continuous setting, each of these drops in entropy can be translated to mutual information.}

\paragraph{KL divergence, entropy, mutual information and their properties}

\begin{definition}[Kullback-Leibler (KL) Divergence]
    \label{def:KL}
    For distributions $\mcD$, $\mcE$ supported on the same domain, the \emph{KL divergence} between $\mcD$ and $\mcE$ is defined as,
    \begin{equation*}
        \KL{\mcD}{\mcE} \coloneqq \Ex_{\bx \sim \mcD}\bracket*{\log \paren*{\frac{\mcD(\bx)}{\mcE(\bx)}}}.
    \end{equation*}
\end{definition}

\begin{definition}[Entropy and conditional entropy]
    \label{def:entropy}
    The entropy of a random variable $\bx \sim \mcD_x$ is defined as
    \begin{equation*}
        H(\bx) \coloneqq \Ex_{\bx \sim \mcD_x}\bracket*{\ln\paren*{\frac{1}{\mcD_{x}(\bx)}}}.
    \end{equation*}
    For a second random variable $\by$ with marginal distribution $\mcD_y$, the conditional entropy is defined,
    \begin{equation*}
        H(\bx \mid \by) \coloneqq \Ex_{\by' \sim \mcD_y}\bracket*{H(\bx \mid \by=\by')}.
    \end{equation*}
\end{definition}

One of the many convenient properties of entropy is that it is \emph{additive}.
\begin{fact}[Additivity of entropy]
    \label{fact:entropy-add}
    For any random variables $\bx, \by$,
    \begin{equation*}
        H((\bx,\by)) = H(\bx) + H(\by\mid \bx) = H(\by)+H(\bx\mid\by).
    \end{equation*}
\end{fact}

\begin{definition}[Mutual information]
    \label{def:MI}
    For random variables $\bx,\by$ jointly distributed according to a distribution $\mcD$, let $\mcD_{x}$ and $\mcD_{y}$ be the marginal distributions of $\bx$ and $\by$ respectively. The mutual information between $\bx$ and $\by$ is defined as
    \begin{equation*}
        I(\bx;\by) \coloneqq H(\bx) - H(\bx \mid \by) = H(\by) - H(\by \mid \bx) = I(\by;\bx),
    \end{equation*}
    which is equivalent to the following
    \begin{equation*}
        I(\bx ; \by) = \KL{\mcD}{\mcD_{x} \times \mcD_{y}} = \Ex_{\by} \bracket*{\KL{\mcD_{x \mid \by}}{\mcD_x}}.
    \end{equation*}
    where $\mcD_{x\mid y}$ is the marginal distribution of $\bx$ conditioned on $\by = y$.
\end{definition}

A useful fact about KL divergence and mutual information is that they are \emph{nonnegative}.
\begin{fact}[Non-negativity of KL divergence and mutual information]
    \label{fact:nonneg}
    For any distributions $\mcD$ and $\mcE$ on the same domain,
    \begin{equation*}
        \KL{\mcD}{\mcE} \geq 0.
    \end{equation*}
    As a consequence, for any random variables $\bx$ and $\by$,
    \begin{equation*}
        I(\bx; \by) \geq 0.
    \end{equation*}
\end{fact}
As an easy corollary, we obtain the following.
\begin{corollary}[Conditioning cannot increase entropy]
    \label{cor:condition-entropy}
    For any random variables $\bx, \by$,
    \begin{equation*}
        H(\bx \mid \by) \leq H(\bx).
    \end{equation*}
\end{corollary}

We'll also use conditional mutual information.
\begin{definition}[Conditional mutual information]
    \label{def:cond-MI}
    For random variables $\bx,\by,\bz$ jointly distributed, the mutual information of $\bx$ and $\by$ conditioned on $\bz$ is
    \begin{equation*}
        I(\bx;\by \mid \bz) = \Ex_{\bz' \sim \mcD_z}\bracket*{I((\bx \mid \bz=\bz');(\by \mid \bz=\bz')}
    \end{equation*}
    where $\mcD_z$ is the marginal distribution of $\bz$.
\end{definition}

\subsection{Formal model of the analyst}
\begin{figure}[hbtp] 

  \captionsetup{width=.9\linewidth}

    \begin{tcolorbox}[colback = white,arc=1mm, boxrule=0.25mm]
    \vspace{2pt} 
    
    \textbf{Input:} A sample $S \in X^n$ not known to the analyst.\vspace{2pt}
    
    For each time step $t \in [T]$, the analyst: \vspace{2pt}
    \begin{enumerate}[nolistsep,itemsep=2pt]
        \item Selects a query $\phi_t: X^{w_t} \to Y_t$ which can depend on the previous responses $\by_1, \ldots, \by_{t-1}$.
        \item Receives  the response $\by_t \sim \phi_t(S)$.
    \end{enumerate}
    \end{tcolorbox}
\caption{An analyst asking an adaptive sequence of subsampling queries.}
\label{fig:adaptive-analyst}
\end{figure} 

\Cref{fig:adaptive-analyst} summarizes the interaction of the analyst with a sample $S \in X^n$. Formally, we model the analyst $\mcA$ as a function mapping a time step $t \in [T]$, previous query responses $y_1, \ldots, y_{t-1}$, and a source of randomness $\bz \sim \mcZ$ to a subsampling query $\phi_t$. After the $T^{\text{th}}$ step, the analyst outputs a series of test queries $\psi_1:X^{v_1} \to \zo, \psi_{m}:X^{v_m} \to \zo$, also as a function of $y_1, \ldots, y_T$ and $\bz$. With this formal description of the analyst, we can give the following lengthy but fully explicit description of \Cref{thm:main-binary}.

\begin{theorem}[\Cref{thm:main-binary} restated with a formal model of the analyst]
    \label{thm:main-analyst}
    For any analyst $\mcA$ and distributions $\mcD, \mcZ$, draw $\bS \sim \mcD^n$ and $\bz \sim \mcZ$. For each $t \in [T]$, let $\bphi_t = \mcA(t, \by_1, \ldots, \by_{t-1}, \bz)$ and $\by_t \sim \phin_t(\bS)$. Then, for $\psi_1, \ldots, \psi_m = \mcA(T,\by_1, \ldots, \by_T, \bz)$,
    \begin{equation}
        \label{eq:expectation-analyst}
        \Ex[\sup_{i \in [m]} \error(\psi_i, \bS, \mcD)] \leq O \paren*{\frac{\Ex\bracket*{\sum_{t \in T}\cost(\bphi_t)}}{n^2} + \frac{\log m + 1}{n}}
    \end{equation}
    where $\cost(\phi:X^w \to Y) \coloneqq w\cdot|Y|$.
\end{theorem}
\subsubsection{Deterministic vs randomized analysts}

We will restrict our attention to \emph{deterministic} analysts, where the analyst's output is a deterministic function of the previous responses. Deterministic analysts do not take in a source of randomness (previously denoted $\bz \sim \mcZ$). Through the following simple argument, this is without loss of generality.
\begin{lemma}[Deterministic analysts are as powerful as randomized analysts]
    \label{lem:rand-to-det}
    \Cref{thm:MI-intro,thm:informal,thm:formal-simple,thm:main-binary,thm:high-probability} hold for deterministic analysts, they also hold for randomized analysts.
\end{lemma}
\begin{proof}
    Let $\mcA$ be a randomized analyst. We can think of it as a mixture of deterministic analysts: $\mcA$ first draws $\bz \sim \mcZ$ and then executes the deterministic strategy $\mcA_{\bz} \coloneqq \mcA(\cdot, \bz)$. We begin with the reduction for \Cref{thm:informal,thm:formal-simple,thm:main-binary},
    \begin{align*}
        \Ex[\text{Error with $\mcA$ as the analyst}] &= \Ex_{\bz}[\Ex[\text{Error with $\mcA_{\bz}$ as the analyst}]] \\
        &\leq \sup_{z} \Ex[\text{Error with $\mcA_{z}$ as the analyst}].
    \end{align*}
    The left-hand side of \Cref{eq:expectation-analyst} when $\mcA$ is the analyst is the expectation of the same quantity when $\mcA_{\bz}$ is the analyst. Therefore, if it is small for all deterministic analysts, it is also small for all randomized analysts.

    For \Cref{thm:high-probability}, the same argument holds where "Error with $\mcA$ as the analyst" is replaced by "Failure probability with $\mcA$ as the analyst." Finally, for \Cref{thm:MI-intro} which bounds the mutual information of $\bS$ and $\by$, let $\by_z$ be the distribution of $\by$ conditioned on the analyst being $\mcA_z$. Then,
    \begin{align*}
         \Ex_{\bz}\bracket*{I(\bS ; \by_{\bz})}- I(\bS ; \by) &=  \paren*{H(\bS)- \Ex_{\bz}\bracket*{H(\bS \mid \by_{\bz})}} - \paren*{H(\bS) -H(\bS \mid \by)} \\
        &= H(\bS \mid \by) - H(\bS \mid \by, \bz) \\
        &= I((\bS \mid \by); \bz) \geq 0. \tag{\Cref{fact:nonneg}}
    \end{align*}
    Therefore, upper bounding the mutual information for deterministic analysts in \Cref{thm:MI-intro} suffices to upper bound it for randomized analysts as well.
\end{proof}

\subsubsection{Deterministic vs randomized subsampling queries}

The analyst may be interested in asking a subsampling query $\phi:X^w \to Y$ that itself incorporates randomness. In this case, we can model $\phi$ as a deterministic function $\phi'$ which also takes in a source of randomness, i.e. $\phi(x_1, \ldots, x_w) \coloneqq \phi'(x_1, \ldots, x_w, \bz)$ where $\bz$ is the source of randomness. However, in each of our theorems, we can without loss of generality assume that the query functions are deterministic by using the random sample as a source of randomness. If the original sample is $\mcD$ we consider the augmented distribution $\mcD' \coloneqq \mcD \times \mcZ_1 \times \mcZ_2 \times \cdots \times \mcZ_T$ where each $\mcZ_t$ will be a source of randomness for the $t^{\text{th}}$ query. We then create a \emph{deterministic} query $\psi_t:(X \times Z_1 \times \cdots \times Z_T)^w \to Y$ that evaluates
\begin{equation*}
    \psi_t(a_1, \ldots, a_w) \coloneqq \phi'((a_1)_x, \ldots, (a_w)_x, (a_1)_{z_t})
\end{equation*}
where for $a \in X \times Z_1 \times \cdots \times Z_T$, the notation $a_x$ takes the first component of $a$ and $a_{z_t}$ takes the $(t+1)^{\text{th}}$ component of $a$. Using $\psi_t$ with the distribution $\mcD'$ will generate the exact same distribution of responses as $\phi_t$ with the distribution $\mcD$, despite the former being a deterministic function and the latter a randomized function.
\section{Average leave-many-out KL stability and its connection to mutual information}
\label{sec:stability-to-MI}

The starting point of our analysis is a new notion of \emph{algorithmic stability}. We require the algorithm's output to be stable even if a large portion of the dataset is removed. This definition generalizes \cite{FS18}'s notion of ``Average leave-one-out KL stability," which is equivalent to the below with $m = 1$.
\begin{definition}[Average leave-many-out KL stability, \Cref{def:ALMOKL-intro} restated]
    \label{def:ALMOKL}
    A randomized algorithm $\mcM: X^n \to Y$ is \emph{ALMOKL stable} with parameters $(m, \eps)$ if there is a randomized algorithm $\mcM': X^{n-m} \to Y$ such that, for all samples $S \in X^n$,
    \begin{equation*}
        \Ex_{\bJ \sim \binom{[n]}{n-m}} \bracket*{\KL{\mcM(S)}{\mcM'(S_{\bJ})}} \leq \eps.
    \end{equation*}
    We'll furthermore use the notation $\stab_m(\mcM)$ to refer to the infimum over $\eps$ such that $\mcM$ is $(m,\eps)$-ALMOKL stable.
\end{definition}
The utility of ALMOKL stability is that we can use it to upper bound the mutual information between a sample and a series of adaptively chosen randomized algorithms.
\begin{theorem}[Using ALMOKL stability to upper bound mutual information, restatement of \Cref{thm:stab-to-MI-intro}]
    \label{thm:stab-to-MI}
    Let $\bS$ drawn from a product distribution over $X^n$ and, for each $t \in [T]$, draw $\by_t \sim \mcM^{\by_1, \ldots, \by_{t-1}}(\bS)$. Then, for any $m \in [n]$,
    \begin{equation*}
        I(\bS; (\by, \ldots, \by_T)) \leq \frac{n}{m} \cdot \Ex_{\by}\bracket*{\sum_{t \in T} \stab_m\paren*{\mcM^{\by_1, \ldots ,\by_{t-1}}}}.
    \end{equation*}
\end{theorem}

The proof of \Cref{thm:stab-to-MI} is broken into two pieces. First, we bound the drop in $m$-conditional entropy in terms of ALMOKL stability. Note that the half-conditional entropy defined in \Cref{eq:def-half-conditional-entropy} corresponds to the $m=n/2$ case below.
\begin{definition}[$m$-conditional entropy]
    \label{def:m-conditional-entropy}
    For any random variable $\bS$ supported on $X^n$ and $m \in [n]$, the $m$-conditional entropy is defined as
    \begin{equation*}
        H_m(\bS) \coloneqq \Ex_{\bJ \sim \binom{[n]}{m}}\bracket*{H(\bS_{\bJ} \mid \bS_{-\bJ})}.
    \end{equation*}
    Furthermore, for any $\by$ with marginal distribution $\mcD_{y}$, we use 
    \begin{equation*}
        H_m(\bS \mid \by) \coloneqq \Ex_{\by' \sim \mcD_{y}}[H_m(\bS \mid \by=\by')].
    \end{equation*}
\end{definition}

\begin{lemma}[ALMOKL stability bounds the drop in $m$-conditional entropy, generalization of \Cref{lem:half-conditional-drop}]
    \label{lem:ALOMKL-to-conditional-body}
    For any $(m,\eps$)-ALMOKL stable $\mcM:X^n \to Y$ and random variable $\bS$ supported on $X^n$ (including non-product distributions),
    \begin{equation*}
        H_m(\bS) - H_m(\bS \mid \mcM(\bS)) \leq \eps.
    \end{equation*}
\end{lemma}

The second ingredient of \Cref{thm:stab-to-MI} is using the drop in $m$-conditional entropy to upper bound mutual information.
\begin{lemma}[$m$-conditional entropy bounds mutual information, generalization of \Cref{lem:half-conditional-to-mutual}]
    \label{lem:conditional-to-mutual-body}
    For any random variables $\bS, \by$ where the distribution of $\bS$ is product on $X^n$ and $m \leq n$,
    \begin{equation*}
        I(\bS; \by) \leq \frac{n}{m} \cdot \paren*{ H_m(\bS) - H_m(\bS \mid \by)}.
    \end{equation*}
\end{lemma}
Before proving each of these two lemmas individually, we'll show how they imply \Cref{thm:stab-to-MI}.
\begin{proof}[Proof of \Cref{thm:stab-to-MI} given \Cref{lem:ALOMKL-to-conditional-body,lem:conditional-to-mutual-body}]
    We first use \Cref{lem:ALOMKL-to-conditional-body} bound the drop in $m$-conditional entropy of $\bS$ by conditioning on $\by \coloneqq (\by_1, \ldots, \by_T)$.
    \begin{align*}
        \Ex_{\by}\bracket*{H_m(\bS) - H_m(\bS \mid \by)} &= \Ex_{\by}\bracket*{\sum_{t \in T}H_m(\bS \mid \by_1, \ldots, \by_{t-1}) - H_m(\bS \mid \by_1, \ldots, \by_{t}) }\\
        &\leq \Ex_{\by}\bracket*{\sum_{t \in T}\stab_{m}\paren*{\mcM^{\by_1, \ldots, \by_{t-1}}}}.
    \end{align*}
    The desired result follows from \Cref{lem:conditional-to-mutual-body}.
\end{proof}

\subsection{Bounding the drop in \texorpdfstring{$m$}{m}-conditional entropy using ALMOKL stability}
In this subsection, we prove \Cref{lem:ALOMKL-to-conditional-body}. This proof uses many of the ideas from Feldman and Steinke's proof of a similar result in the special case where $m = 1$ \cite{FS18}. Just like them, we use the following fact that states the ``center" of a collection of probability, as measured by minimizing the average KL divergence to each distribution, is exactly the mean of those distributions.
\begin{fact}[\cite{BMDG05} Proposition 1 \& \cite{FSG08} Theorem II.1]
    \label{fact:KL-mean}
    Let $\{\mcD_a\}$ be a family of distributions, each supported on $X$, and indexed by $a \in A$ and $\mcA$ be a distribution supported on $A$. Let $\mcD_{\mcA} = \Ex_{\ba \sim \mcA}[\mcD_{\ba}]$ denote the convex combinations of the distributions $\set{\mcD_a}$ with mixture weights given by $\mcA$. Then, for any $\mcD'$ supported on $X$,
    \begin{equation*}
        \Ex_{\ba \sim \mcA}\bracket*{\KL{\mcD_{\ba}}{\mcD_{\mcA}}} \leq \Ex_{\ba \sim \mcA}\bracket*{\KL{\mcD_{\ba}}{\mcD'}}.
    \end{equation*}
\end{fact}
\Cref{fact:KL-mean} actually holds for any \emph{Bregman divergence}, which, in particular, includes KL divergence.
\begin{proof}[Proof of \Cref{lem:ALOMKL-to-conditional-body}]
    Expanding the definitions, our goal is to show that for any $\bS$ distributed on $X^n$, $m \in [n]$, and randomized algorithms $\mcM:X^n \to Y$, $\mcM':X^{n-m} \to Y$,
    \begin{equation*}
        \Ex_{\bJ \sim \binom{[n]}{m}}\bracket*{H(\bS_{\bJ} \mid \bS_{-\bJ}) - H(\bS_{\bJ} \mid \bS_{-\bJ}, \mcM(\bS))} \leq \sup_{S \in X^n}\paren*{\Ex_{\bJ \sim \binom{[n]}{m}}\bracket*{\KL{\mcM(S)}{\mcM'(S_{-\bJ})}}}.
    \end{equation*}
    Rather than picking a worst-case $S$, we'll prove the above holds on average over the distribution of $\bS$:
    \begin{align*}
        \Ex_{\bJ \sim \binom{[n]}{m}}\bracket*{H(\bS_{\bJ} \mid \bS_{-\bJ}) - H(\bS_{\bJ} \mid \bS_{-\bJ}, \mcM(\bS))} &\leq \Ex_{\bS}\bracket*{\Ex_{\bJ \sim \binom{[n]}{m}}\bracket*{\KL{\mcM(\bS)}{\mcM'(\bS_{-\bJ})}}} \\
        &= \Ex_{\bJ \sim \binom{[n]}{m}}\bracket*{\Ex_{\bS}\bracket*{\KL{\mcM(\bS)}{\mcM'(\bS_{-\bJ})}}}
    \end{align*}
    We'll show that the above inequality holds for each $J \in \binom{[n]}{m}$ individually, which implies the desired result by averaging over the choices of $J$. Throughout this proof, we'll use $\Pr[\cdot]$ as an operator that takes as input a random variable and outputs the probability a new draw of that random variable takes on that value (for example, $H(\bx) = \Ex_{\bx}[\ln(1/\Pr[\bx])]$). Let $\by \sim \mcM(\bS)$. Then,
    \begin{align*}
        H(\bS_{\bJ} \mid \bS_{-\bJ}) - H(\bS_{\bJ} \mid \bS_{-\bJ}, \mcM(\bS)) &= \Ex_{\bS, \by}\bracket*{\ln\paren*{\frac{1}{\Pr[\bS_{J} \mid \bS_{-J}]}}-\ln\paren*{\frac{1}{\Pr[\bS_{J} \mid \bS_{-J},\by]}}}\\
        &=\Ex_{\bS, \by}\bracket*{\ln\paren*{\frac{\Pr[\bS_{J} \mid \bS_{-J},\by]}{\Pr[\bS_{J} \mid \bS_{-J}]}}}
    \end{align*}
    Note that, for any random variables $\ba,\bb,\bc$,
    \begin{equation*}
        \frac{\Pr[\ba\mid\bb,\bc]}{\Pr[\ba\mid\bc]} = \frac{\Pr[\ba\mid\bb,\bc]}{\Pr[\ba\mid\bc]} \cdot \frac{\Pr[\bb\mid\bc]}{\Pr[\bb\mid \bc ]}= \frac{\Pr[\ba, \bb\mid\bc]}{\Pr[\ba\mid\bc]\cdot \Pr[\bb\mid\bc]} = \frac{\Pr[\bb\mid\ba,\bc]}{\Pr[\bb\mid\bc]}.
    \end{equation*}
    Therefore,
    \begin{align*}
         H(\bS_{\bJ} \mid \bS_{-\bJ}) - H(\bS_{\bJ} \mid \bS_{-\bJ}, \mcM(\bS)) &= \Ex_{\bS, \by}\bracket*{\ln\paren*{\frac{\Pr[\by \mid \bS]}{\Pr[\by \mid \bS_{-J}]}}} \\
         &= \Ex_{\bS}\bracket*{\KL{(\by \mid \bS)}{(\by \mid \bS_{-J})}}\\
         &= \Ex_{\bS_{-J}}\bracket*{\Ex_{\bS_J \mid \bS_{-J}}\bracket*{\KL{(\by \mid \bS_{-J}, \bS_J)}{(\by \mid \bS_{-J})}}}.
    \end{align*}
    Here, we utilize \Cref{fact:KL-mean}. Fix any choice of $\bS_{-J} = S_{-J}$. The distribution of $(\by \mid \bS_{-J} =S_J)$ is exactly the convex combination of $\by \mid \bS_{J}, S_{-J}$ where each $\bS_{J}$ is sampled conditioned on $\bS_{-J}$. Therefore, it is the distribution minimizing the above KL divergence. As a result, we can replace it with any other distribution, namely $\mcM'(\bS_{-J})$, without decreasing the KL divergence.
    \begin{align*}
        H(\bS_{\bJ} \mid \bS_{-\bJ}) - H(\bS_{\bJ} \mid \bS_{-\bJ}, \mcM(\bS)) &\leq \Ex_{\bS_{-J}}\bracket*{\Ex_{\bS_J \mid \bS_{-J}}\bracket*{\KL{(\by \mid \bS_{-J}, \bS_J)}{\mcM'(\bS_{-J}}}} \\
        &= \Ex_{\bS}\bracket*{\KL{\mcM(\bS)}{\mcM'(\bS_{-J}}}. \qedhere
    \end{align*}
\end{proof}

\subsection{Bounding mutual information using \texorpdfstring{$m$}{m}-conditional entropy}
In this subsection, we prove \Cref{lem:conditional-to-mutual-body} which bounds the mutual information between $\bS$ and $\by$ in terms of the average drop in $m$-conditional of $\bS$ conditioned on $\by$. We begin with the special case where $n$ is even and $m = n/2$, as it is simpler and furthermore sufficient to recover all of our results with the mild assumption that the sample size is even.

\begin{proposition}
    \label{prop:half-conditional}
    For any $S \in X^{2m}$ and $J \in \binom{[2m]}{m}$,
    \begin{equation*}
        H(\bS) = 2H_m(\bS) + \Ex_{\bJ \sim \binom{[2m]}{m}}[I(\bS_{\bJ}; \bS_{-\bJ})].
    \end{equation*}
\end{proposition}
\begin{proof}
    For any choice of $J \in \binom{[2m]}{m}$,
    \begin{align*}
        H(\bS) &= H(\bS_{-J}) + H(\bS_{\bJ} \mid \bS_{-J}) \tag{\Cref{fact:entropy-add}}\\
        &= \paren*{H(\bS_{-J} \mid \bS_{J}) + I(\bS_{-J}; \bS_{J})}+ H(\bS_{\bJ} \mid \bS_{-J})\tag{\Cref{def:MI}}.
    \end{align*}
    The desired result follows by averaging over all choices of $\bJ$ and noting that the distribution of $\bJ$ and its compliment are identical.
\end{proof}

\begin{proof}[Proof of \Cref{lem:conditional-to-mutual-body} in the special case where $m = n/2$]
    We bound,
    \begin{align*}
        I(\bS; \by) &= H(\bS) - H(\bS \mid \by) \tag{\Cref{def:MI}} \\
        &= 2H_m(\bS) + \Ex_{\bJ \sim \binom{[2m]}{m}}[I(\bS_{\bJ}; \bS_{-\bJ})] - \paren*{2H_m(\bS \mid \by) + \Ex_{\bJ \sim \binom{[2m]}{m}}[I(\bS_{\bJ}; \bS_{-\bJ}\mid\by)]} \tag{\Cref{prop:half-conditional}}\\
        &= 2H_m(\bS) - 2H_m(\bS \mid \by) - \Ex_{\bJ \sim \binom{[2m]}{m}}[I(\bS_{\bJ}; \bS_{-\bJ}\mid\by)] \tag{$\bS$ is product} \\
        &\leq 2H_m(\bS) - 2H_m(\bS \mid \by)\tag{Nonnegativity of mutual information, \Cref{fact:nonneg}}
    \end{align*}
    This is exactly the special case of \Cref{lem:conditional-to-mutual-body} where $m = n/2$.
\end{proof}

To expand beyond this special case, we will appropriately generalize \Cref{prop:half-conditional}. We separate it into two cases: When $\bS$ is product and when it is not.
\begin{lemma}
    \label{lem:conditional-product}
    If $\bS$ is drawn from a product distribution over $X^n$,
    \begin{equation*}
        H(\bS) = \frac{n}{m} \cdot H_m(\bS).
    \end{equation*}    
\end{lemma}
\begin{proof}
    We can expand $H(\bS)$ using the additivity of entropy, where $S_{<i} \coloneqq [S_1, \ldots, S_{i-1}]$,
    \begin{align*}
        H(\bS) &= \sum_{i \in [n]}H(\bS_i \mid \bS_{< i}) \\
        &= \sum_{i \in [n]}H(\bS_i). \tag{$\bS$ is product}
    \end{align*}
    Similarly, we expand $H_m(\bS)$,
    \begin{align*}
        H_m(\bS) &= \Ex_{\bJ \sim \binom{[n]}{m}}[H(\bS_{\bJ} \mid \bS_{-\bJ})] \\
        &= \Ex_{\bJ \sim \binom{[n]}{m}}\bracket*{\sum_{j \in J}H(\bS_j)}.
    \end{align*}
    The desired result follows from each index appearing in $\bj$ with probability $\frac{m}{n}$.
\end{proof}
\begin{lemma}
    \label{lem:conditional-non-product}
    For any $\bS$ distributed on $X^n$ (including non-product distributions),
    \begin{equation*}
        H(\bS) \geq \frac{n}{m} \cdot H_m(\bS).
    \end{equation*}
\end{lemma}
Taken together, \Cref{lem:conditional-product,lem:conditional-non-product} show that $H(\bS) - n/m \cdot H_m(\bS)$ is a measure of the distance from $\bS$ to a product distribution. 

The proof of \Cref{lem:conditional-non-product} uses the following.
\begin{proposition}
    \label{lem:one-conditional}
    For any $\bS$ distribution on $X^n$ and $k \leq n-m$,
    \begin{equation}
        \label{eq:one-conditional}
        \Ex_{\substack{\bK \sim \binom{[n]}{k} \bj \sim ([n]\setminus \bK)}}[H(\bS_{\bj} \mid \bS_{\bK})] \geq \frac{H_m(\bS)}{m}.
    \end{equation}
\end{proposition}
\begin{proof}
    First, since conditioning cannot increase entropy (\Cref{cor:condition-entropy}), the lefthand side of \Cref{eq:one-conditional} is non-increasing in $k$. Therefore it suffices to consider the case where $k = n-m$. By the additivity of entropy, for any $J \coloneqq (j_1, \ldots, j_{m}) \in \binom{n}{[m]}$,
    \begin{align*}
        H(\bS_{J} \mid \bS_{-J}) &= \sum_{i \in [m]}H(S_{j_i} \mid \bS_{-J}, \bS_{j_1}, \ldots, \bS_{j_{i-1}}) \\
        & \leq \sum_{i \in [m]}H(S_{j_i} \mid \bS_{-J}) \tag{\Cref{cor:condition-entropy}}
    \end{align*}
    Dividing both sides by $m$ and averaging over all choices of $\bJ = J$ exactly recovers the desired result.
\end{proof}

\begin{proof}[Proof of \Cref{lem:conditional-non-product}]
    Fix a single choice $J \in \binom{n}{[m]}$ and let $k_1, \ldots, k_{n-m}$ be the elements in its compliment. Then, by the additivity of entropy,
    \begin{equation*}
        H(\bS) = \paren*{\sum_{i \in [n-m]} H(\bS_{k_i} \mid \bS_{k_1},\ldots, \bS_{k_{i-1}})} + H(\bS_J \mid \bS_{-J}).
    \end{equation*}
    Averaging over all choices of $J$ and all permutations of $k_1, \ldots, k_{n-m}$,
    \begin{align*}
        H(\bS) &= \paren*{\sum_{i \in [n-m]} \Ex_{\substack{\bK \sim \binom{[n]}{i-1} \bj \sim ([n]\setminus \bK)}}[H(\bS_{\bj} \mid \bS_{\bK})]} + H_m(\bS) \\
        &\geq (n-m)\cdot \frac{H_m(\bS)}{m} + H_m(\bS) \tag{\Cref{lem:one-conditional}} \\
        &= \frac{n}{m} \cdot H_m(\bS). \qedhere
    \end{align*}
\end{proof}
Finally, we prove \Cref{lem:conditional-to-mutual-body}.
\begin{proof}[Proof of \Cref{lem:conditional-to-mutual-body}]
    We bound,
    \begin{align*}
        I(\bS; \by) &= H(\bS) - H(\bS \mid \by) \tag{\Cref{def:MI}} \\
        &= \frac{n}{m}\cdot H_m(\bS) - H(\bS \mid \by) \tag{\Cref{lem:conditional-product}, $\bS$ is product}\\
        &\leq \frac{n}{m}\cdot H_m(\bS) - \frac{n}{m} \cdot H_m(\bS \mid \by) \tag{\Cref{lem:conditional-non-product}}
    \end{align*}
    which is exactly the desired result.
\end{proof}
\section{Bounding the stability of subsampling queries}

The goal of this section is prove the following result.
\begin{lemma}
    \label{lem:ALMOKL-stab}
    For any function $\phi:X^w \to Y$ and $m \leq n$, the randomized algorithm which maps $S \in X^n$ to $\phi(S)$ is $(m, \eps)$-ALMOKL stable for
    \begin{equation*}
        \eps \coloneqq \frac{|Y|-1}{\floor{\frac{n-m}{w}} + 1} \leq \frac{w(|Y|-1)}{n-m + 1}.
    \end{equation*}
\end{lemma}
Before proving \Cref{lem:ALMOKL-stab}, we show how it gives a bound on the mutual information between the sample and query responses.
\begin{theorem}[\Cref{thm:MI-intro} restated]
    \label{thm:MI-general}
    For any distribution $\mcD$ over domain $X$ and analyst that makes a series of adaptive queries $\bphi_1:X^{\bw_1} \to \bY_1, \ldots, \bphi_T:X^{\bw_T} \to \bY_T$ to a sample $\bS \sim \mcD^n$ and receives responses $\by \coloneqq (\by_1, \ldots, \by_T)$
    \begin{equation*}
        I(\bS; \by) \leq\frac{4\Ex_{\bphi_1, \ldots, \bphi_T}\bracket*{\sum_{t \in T}\bw_t |\bY_t|}}{n}.
    \end{equation*}
\end{theorem}
\begin{proof}[Proof of \Cref{thm:MI-general} assuming \Cref{lem:ALMOKL-stab}]
    We set $m \coloneqq \ceil{n/2}$. By \Cref{lem:ALMOKL-stab}, the $m$-ALMOKL stability of a subsampling query $\phi:X^w \to Y$ is $\frac{w(|Y| - 1)}{n-m+1} \leq \frac{2w(|Y|-1)}{n}$. The desired result follows from \Cref{thm:stab-to-MI} and that $n/m \leq 2$.
\end{proof}
 
To prove \Cref{lem:ALMOKL-stab}, we need to design a randomized $\mcM'_{\phi}$ for which,
\begin{equation*}
    \Ex_{\bJ \sim \binom{[n]}{n-m}} \bracket*{\KLbig{\phi(S)}{\mcM'_{\phi}(S_{\bJ})}} \leq \eps.
\end{equation*}
A natural choice for $\mcM'_{\phi}$ is to make it perform the subsampling query $\phi$, meaning it outputs $y$ with probability $\Prx_{\bx_1, \ldots, \bx_w \sim (S_{J})^{(w)}}[\phi(\bx_1, \ldots, \bx_w) = y]$. Unfortunately, this natural choice can lead to an infinite KL-divergence, even in the simplest case where $w = 1$. If $S$ has only contains a single point that evaluates to $y$, then with nonzero probability $S_{\bJ}$ will contain \emph{no} points that evaluate to $y$. In that case, the support of $\phi(S)$ will be strictly smaller than that of $\phi(S)$, leading to an infinite KL divergence.

To avoid this, we smooth $\mcM'_{\phi}$ to ensure it has a nonzero probability of outputting every point in $Y$. Specifically, we set,
\begin{align}
    \label{eq:def-noisy}
    \mcM'_{\phi}(S_J) &\coloneqq \begin{cases}
        \Unif(Y) &\text{wp }\alpha\\
        \phi(S_J)&\text{wp }1-\alpha
    \end{cases} &\text{where},\quad \alpha \coloneqq \frac{|Y|}{|Y| + \floor{\frac{n-m}{w}}}.
\end{align}
Our choice of $\alpha$ is purely for technical convenience. As we will see in the proof of \Cref{lem:ALMOKL-stab} that with this choice of $\alpha$ we are able to use the following inverse moment bound for the binomial distribution.
\begin{fact}[Inverse expectation of a binomial random variable]
    \label{fact:inv-exp}
    For any $n \in \N$ and $p \in (0,1)$,
     \begin{equation*}
         \Ex_{\bk \sim \Bin(n,p)}\bracket*{\frac{1}{\bk + 1}} = \frac{1 - (1-p)^{n+1}}{p(n+1)} \leq \frac{1}{p(n+1)}.
     \end{equation*}
\end{fact}
\begin{proof}
    We expand the desired expectation,
    \begin{align*}
        \Ex\bracket*{\frac{1}{\bk + 1}} &= \sum_{k = 0}^n \frac{\binom{n}{k} p^k (1-p)^{n-k}}{k+1}\\
        &= \frac{1}{n+1} \cdot  \sum_{k = 0}^n \binom{n+1}{k+1}p^k (1-p)^{n-k} \\
        &= \frac{1}{p(n+1)} \cdot \sum_{k = 0}^n \binom{n+1}{k+1}p^{k+1} (1-p)^{(n+1)-(k+1)} \\
        &= \frac{1}{p(n+1)} \cdot \Prx_{\bk' \sim \Bin(n+1, p)}[\bk' \neq 0] \\
        &= \frac{1 - (1-p)^{n+1}}{p(n+1)}. \qedhere
    \end{align*}
\end{proof}
Other choices of $\alpha$ in \Cref{eq:def-noisy} would correspond to different inverse moments of the binomial, which require a more involved analysis than the above simple fact (see e.g. \cite{CS72}). Our proof of \Cref{lem:ALMOKL-stab} uses our generalization of Hoeffding's reduction theorem, \Cref{thm:u-stat-convex}, which we prove in \Cref{sec:hoeffding}.

\begin{proof}[Proof of \Cref{lem:ALMOKL-stab}]
    We use the mechanism $\mcM'_{\phi}$ defined in \Cref{eq:def-noisy}. Our task is to bound, for any sample $S \in X^n$,
    \begin{equation*}
        \Ex_{\bJ \sim \binom{[n]}{n-m}} \bracket*{\KLbig{\phi(S)}{\mcM'_{\phi}(S_{\bJ})}} = \Ex_{\bJ \sim \binom{[n]}{n-m}}\bracket*{\Ex_{\by \sim \phi(S)}\bracket*{\ln\paren*{\frac{\Pr[\phi(S) = \by]}{\Pr[\mcM'_{\phi}(S_{\bJ}) = \by]}}}}.
    \end{equation*}
    For each $y \in Y$ and $J \in \binom{[n]}{n-m}$, we'll define $q_y(S_J) \coloneqq \Pr\bracket*{\phi(S_{J}) = y}$. Then, for $\ell \coloneqq \floor{\frac{n-m}{w}}$,
    \begin{equation*}
         \Pr\bracket*{\mcM'_{\phi}(S_{J}) = y} = \frac{1}{\ell + |Y|} + \frac{\ell}{\ell + |Y|}\cdot q_y(S_J) = \frac{\ell \cdot q_y(S_J) + 1}{\ell + |Y|}.
    \end{equation*}
    We'll also use $p_y(S)$ as shorthand for $\Pr[\phi(S) = y]$. With this notation,
    \begin{align*}
        \Ex_{\bJ \sim \binom{[n]}{n-m}} \bracket*{\KLbig{\phi(S)}{\mcM'_{\phi}(S_{\bJ})}} &=  \Ex_{\by \sim \phi(S)}\bracket*{\Ex_{\bJ \sim \binom{[n]}{n-m}}\bracket*{\ln\paren*{\frac{p_{\by}(S)}{\frac{\ell \cdot q_{\by}(S_{\bJ}) + 1}{\ell + |Y|}}}}}\\
        &=  \Ex_{\by \sim \phi(S)}\bracket*{\Ex_{\bJ \sim \binom{[n]}{n-m}}\bracket*{\ln\paren*{\frac{p_{\by}(S)(\ell + |Y|)}{\ell \cdot q_{\by}(S_{\bJ}) + 1}}}}.
    \end{align*}
    Here, we apply \Cref{thm:u-stat-convex}. For any $y \in Y$, let $\phi_y:X^w \to [0,1]$ be function $T \mapsto \Ind[\phi(T) = y]$ which indicates whether the subsampling query returns $y$. We choose $f_y(t) = \ln(\frac{p_y(S)(\ell + |Y|)}{\ell \cdot t + 1})$ as the convex function. Then, \Cref{thm:u-stat-convex} says,
    \begin{equation*}
        \Ex_{\bJ \sim \binom{[n]}{n-m}}\bracket*{\ln\paren*{\frac{p_{\by}(S)(\ell + |Y|)}{\ell \cdot q_{\by}(S_{\bJ}) + 1}}} = \Ex_{\bJ \sim \binom{[n]}{n-m}}\bracket*{f_y\paren*{\mu_{\phi}(S_{\bJ})}} \leq \Ex_{\bT_1, \ldots, \bT_{\ell} \iid S^{(w)}}\bracket*{f\paren*{\frac{1}{\ell}\sum_{i \in [k]}\phi_y(\bT_i)}}
    \end{equation*}
    Based on the definition of $p_y(S)$, we have that $\Pr[\phi_y(\bT_i) = 1] = p_y(S)$. Therefore, we can continue to bound,
    \begin{align*}
        \Ex_{\bJ \sim \binom{[n]}{n-m}} \bracket*{\KLbig{\phi(S)}{\mcM'_{\phi}(S_{\bJ})}} &\leq \Ex_{\by \sim \phi(S)}\bracket*{\Ex_{\bz \sim \Bin(\ell, p_{\by}(S))}\bracket*{\ln\paren*{\frac{p_{\by}(S)(\ell + |Y|)}{\bz + 1}}}} \tag{\Cref{thm:u-stat-convex}} \\
        &\leq \Ex_{\by \sim \phi(S)}\bracket*{\Ex_{\bz \sim \Bin(\ell, p_{\by}(S))}\bracket*{\frac{p_{\by}(S)(\ell + |Y|)}{\bz + 1} - 1}} \tag{$\ln x \leq x-1$} \\
        &\leq \Ex_{\by \sim \phi(S)}\bracket*{\Ex_{\bz \sim \Bin(\ell, p_{\by}(S))}\bracket*{\frac{p_{\by}(S)(\ell + |Y|)}{p_{\by}(S)(\ell + 1)} - 1}} \tag{\Cref{fact:inv-exp}}\\
        &= \frac{|Y| - 1}{\ell +1}
    \end{align*}
    which is exactly the desired bound based on our definition of $\ell \coloneqq\floor{\frac{n-m}{w}}$.
\end{proof}
\subsection{Improving Hoeffding's reduction theorem}
\label{sec:hoeffding}
In this section, we prove \Cref{thm:u-stat-convex}.
Recall that, for any $S \in X^n$ and $w \leq n$, we use $\bT \sim S^{(w)}$ to indicate that $\bT$ contains $w$ element from $S$ sampled uniformly without replacement, and, for any $\phi:X^w \to \R$, we use $\mu_{\phi}(S) \coloneqq \Ex_{\bT \sim S^{(w)}}[\phi(\bT)]$.

\begin{theorem}[\Cref{thm:u-stat-convex} restated]
    \label{thm:u-stat-convex-body} 
      Consider any integers $w \leq m \leq n$, finite population $S \in X^n$, and $\phi:X^w \to \R$. For any convex $f: \R \to \R$,
     \begin{equation*}
         \Ex_{\bJ \sim \binom{[n]}{m}}\bracket*{f\paren*{\mu_{\phi}(S_{\bJ})}} \leq \Ex_{\bT_1, \ldots, \bT_k \iid S^{(w)}}\bracket*{f\paren*{\frac{1}{k}\sum_{i \in [k]}\phi(\bT_i)}}
     \end{equation*}
     where $k = \floor{m/w}$.
\end{theorem}
\begin{proof}
    Without loss of generality, and to ease notation, we will assume that $X = [n]$. Otherwise, for $X = (x_1, \ldots, x_n)$, we could consider the kernel $\psi(i_1, \ldots, i_w) \coloneqq \phi(x_{i_1}, \ldots, x_{i_w})$, because the U-statistic with kernel $\phi$ and population $X$ has the same distribution as that with kernel $\psi$ and population $[n]$.

    We draw $\bT_1, \ldots, \bT_k \iid [n]^{(w)}$ and note that at most $wk \leq m$ unique elements are contained in $\bT_1 \cup \cdots \cup \bT_k$. Then, we draw $\bJ$ uniform from $\binom{[n]}{m}$ conditioned on all the those unique elements being contained in $\bJ$. By symmetry, we observe the following:
    \begin{enumerate}
        \item Marginally (i.e. \emph{not} conditioning on the values of $\bT_1, \ldots, \bT_k$), the distribution of $\bJ$ is uniform over $\binom{[n]}{m}$.
        \item Conditioning on any choice of $\bJ = J$, the distribution of $\bT_i$ is uniform over $J^{(w)}$ for every $i \in [k]$.
    \end{enumerate}

    We proceed to bound,
    \begin{align*}
        \Ex_{\bJ}\bracket*{f\paren*{\mu_{\phi}(X_{\bJ})}} &= \Ex_{\bJ}\bracket*{f\paren*{\mu_{\phi}(\bJ)}} \tag{WLOG, $X = [n]$}\\
         &= \Ex_{\bJ}\bracket*{f\paren*{\Ex_{\bT' \sim \bJ^{(w)}}[\phi(\bT')]}} \tag{Definition of $\mu_\phi(\bJ)$}\\
         &= \Ex_{\bJ}\bracket*{f\paren*{\Ex_{\bT_i \mid \bJ}[\phi(\bT_i)]}} \quad \text{for any $i\in [k]$} \tag{$\bT_i\mid \bJ$ is uniform over $\bJ^{(w)}$}\\
         &= \Ex_{\bJ}\bracket*{f\paren*{\Ex_{\bT \mid \bJ}\bracket*{\frac{1}{k}\sum_{i \in [k]}\phi(\bT_i)}}} \tag{Linearity of expectation}\\
         &\leq \Ex_{\bJ}\bracket*{\Ex_{\bT \mid \bJ}\bracket*{f\paren*{\frac{1}{k}\sum_{i \in [k]}\phi(\bT_i)}}} \tag{Jensen's inequality}\\
         &= \Ex_{\bT}\bracket*{f\paren*{\frac{1}{k}\sum_{i \in [k]}\phi(\bT_i)}}, \tag{Law of total expectation}
    \end{align*}
    which, since we drew $\bT_1, \ldots, \bT_k \iid [n]^{(w)}$, is exactly the desired result.
    
\end{proof}

\section{From small mutual information to small bias}
\label{sec:gen}

In this section, we connect the mutual information bound of \Cref{thm:MI-intro} to generalization error, completing the proof of \Cref{thm:informal,thm:formal-simple,thm:main-binary}. Recall our definition of \emph{error} for a test query.
\begin{definition}[\Cref{def:error-simple}, restated]
    \label{def:error-second}
    For any $\psi:X^w \to [0,1]$ and distribution $\mcD$ over $X$, we define
     \begin{equation*}
        \error(\psi, S, \mcD) \coloneqq \frac{1}{w} \cdot \min\paren*{\Delta, \frac{\Delta^2}{\sigma^2}
        }
    \end{equation*}
    where
    \begin{equation*}
        \Delta \coloneqq \abs{\mu_{\psi}(S) - \mu_{\psi}(\mcD)} \quad\quad\text{and}\quad\quad \sigma^2 \coloneqq \Varx_{\bS \sim \mcD^w}[\psi(\bS)].
    \end{equation*}
\end{definition}

\begin{theorem}[Mutual information bounds bias]
    \label{thm:MI-bounds-bias}
    For any $\bS \sim \mcD^n$ and $\by \in Y$, as well as $\Psi^{(\by)}$ a set of test queries each mapping $X^* \to [0,1]$ chosen as a function of $\by$,
    \begin{equation*}
        \Ex_{\bS, \by}\bracket*{\sup_{\psi \in \Psi^{(\by)}} \error(\psi, \bS, \mcD)} \leq \frac{16I(\bS; \by) + 8\Ex_{\by}\bracket*{\ln\abs*{\Psi^{(\by)}}}+ 24 \ln 2}{n}.
    \end{equation*}
\end{theorem}

Our proof will use the following fact.
\begin{fact}[\cite{FS18}]
    \label{fact:gen}
    For any random variables $\bS \in X$ and $\bpsi: X \to \R \in \Psi$, as well as $\lambda > 0$,
    \begin{equation*}
        \Ex_{\bS, \by}\bracket*{\bpsi(\bS)} \leq \frac{1}{\lambda}\paren*{I(\bS; \bpsi) + \sup_{\psi \in \Psi}\ln \paren*{\Ex_{\bS}\bracket*{\exp\paren*{\lambda \psi(\bS)}}}}.
    \end{equation*}
\end{fact}

\begin{corollary}
    \label{cor:MI-expectation}
    For any random variables $\bx \in X$ and $\boldf:X \to \R \in F$ satisfying, for all $t \geq 0$ and $f \in F$, $\Pr_{\bx}[f(\bx) \geq t] \leq e^{-t}$,
    \begin{equation*}
        \Ex[\boldf(\bx)] \leq 2(I(\bx; \boldf) + \ln 2).
    \end{equation*}
\end{corollary}
\begin{proof}
    Fix an arbitrary $f \in F$. We first bound the moment-generating function of $f(\bx)$.
    \begin{align*}
        \Ex[\exp(f(\bx)/2)] &= \int_{0}^\infty \Pr[\exp(f(\bx)/2) \geq t]dt\\
        &\leq 1 + \int_{1}^\infty \Pr[f(\bx) \geq 2\ln t]dt \\
        &\leq 1 + \int_{1}^\infty t^{-1/2}dt = 2.
    \end{align*}
    The desired result follows from \Cref{fact:gen} with $\lambda = 1/2$.
\end{proof}

\emph{Bernstein's inequality} will give the sorts of high probability bounds needed to apply \Cref{cor:MI-expectation}.

\begin{fact}[Bernstein's inequality]
    \label{fact:bernstein}
    For any iid mean-zero random variables $\ba_1, \ldots, \ba_n$ with variance $\sigma^2$ and satisfying, almost surely, that $|\ba_i| \leq 1$, let $\bA = \frac{1}{n} \cdot \sum_{i \in [n]} \ba_i$. Then,
    \begin{equation}
        \label{eq:berstein-bound}
        \Pr\bracket*{|\bA| \geq \Delta} \leq 2\exp\paren*{-\frac{\Delta^2n}{2(\sigma^2 +\frac{\Delta}{3})}}.
    \end{equation}
\end{fact}

For our setting, a black-box application of Bernstein's inequality is not sufficient. We wish to prove concentration of a random variable that is \emph{not} necessarily the sum of iid random variables. Fortunately, the proof of Bernstein's inequality only uses a bound on the moment generating function of $\bA$: It proceeds by applying Markov's inequality to the random variable $e^{\lambda \bA}$ for an appropriate choice of $\lambda \in \R$. As a result, the following also holds.
\begin{fact}[Generalization of Bernstein's inequality]
    \label{fact:Bernstein-MGF} Let $\bB$ be any random variable satisfying, for all $\lambda \in \R$,
    \begin{equation*}
        \Ex[e^{\lambda \bB}] \leq \Ex[e^{\lambda \bA}],
    \end{equation*}
    where $\bA$ is as in \Cref{fact:bernstein}. Then, the bound of \Cref{eq:berstein-bound} also holds with $\bB$ in place of $\bA$.
\end{fact}

We'll use the following to produce a bound in our setting.

\begin{proposition}
    \label{prop:convex}
    For any $\psi:X^w \to \R$, distribution $\mcD$ over $X$, and convex function $f: \R \to \R$, set $k = \floor{n/w}$. Then,
    \begin{equation*}
        \Ex_{\bS \sim \mcD^n}\bracket*{f\paren*{\mu_{\psi}(\bS)}} \leq \Ex_{\bT_1, \ldots, \bT_k \iid \mcD^w}\bracket*{f\paren*{\Ex_{\bi \sim [k]}\bracket*{\psi(\bT_i)]}}}.
    \end{equation*}
\end{proposition}
\Cref{prop:convex} is ``almost" a simple corollary of \Cref{thm:u-stat-convex}. Suppose we first sampled $\bX \sim \mcD^N$ for some large $N$ and then $\bS' \sim \bX^{(n)}$. As $N$ grows, the distribution of $\bS'$ approaches that of $\bS$. Similarly, if we draw $\bT_1', \ldots, \bT_k' \iid \bX^{(w)}$, then the distribution of $(\bT_1, \ldots, \bT_k)$ is equivalent to that of $(\bT_1', \ldots, \bT_k')$. By comparing the statement of \Cref{prop:convex} to that of \Cref{thm:u-stat-convex}, we see that they become equivalent as $N$ goes to infinity. To avoid dealing with this limiting behavior, we instead give a direct proof of \Cref{prop:convex}.
\begin{proof}
    We'll couple the distributions of $\bT_1, \ldots, \bT_k$ and $\bS$. First, draw $\bT_1, \ldots, \bT_k \iid \mcD^w$ and $\bT_{k+1} \sim \mcD^{n - kw}$. Then concatenate these $(k+1)$ groups of points and uniformly permute all $n$ points to form $\bS$. We observe that
    \begin{enumerate}
        \item The distribution of $\bS$ is uniform over $\mcD^n$, as it is in the statement of \Cref{prop:convex}.
        \item Conditioning on the value of $\bS = S$, the distribution of $\bT_i$ is uniform over $S^{(w)}$ for any $i \in [k]$.
    \end{enumerate}
    Finally, we bound,
    \begin{align*}
        \Ex_{\bS \sim \mcD^n}\bracket*{f\paren*{\mu_{\psi}(\bS)}} &= \Ex_{\bS \sim \mcD^n}\bracket*{f\paren*{\Ex_{\bT \mid \bS}[\psi(\bT_i)]}} \quad\text{for any }i \in [k] \tag{$\bT_i$ uniform in $\bS^{(w)}$} \\
        &= \Ex_{\bS \sim \mcD^n}\bracket*{f\paren*{\Ex_{\bT \mid \bS}\bracket*{\Ex_{\bi \sim [k]}[\psi(\bT_{\bi})]}}}  \tag{Linearity of expectation} \\
         &\leq \Ex_{\bS \sim \mcD^n}\bracket*{\Ex_{\bT \mid \bS}\bracket*{f\paren*{\Ex_{\bi \sim [k]}[\psi(\bT_{\bi})]}}}  \tag{Jensen's inequality} \\
         &=\Ex_{\bT_1, \ldots, \bT_k \iid \mcD^w}\bracket*{f\paren*{\Ex_{\bi \sim [k]}\bracket*{\psi(\bT_i)]}}} \tag{Law of total expectation}
    \end{align*}

\end{proof}
Since $x \mapsto \exp{\lambda x}$ is convex, using \Cref{fact:Bernstein-MGF}, we arrive at the following corollary.
\begin{corollary}
    \label{cor:bernstein}
    For any $\psi:X^w \to [-1,1]$ and distribution $\mcD$ over $X$ such that $\Ex_{\bS \sim \mcD^w}[\psi(\bS)] = 0$ and $\Varx_{\bS \sim \mcD^w}[\psi(\bS)] = \sigma^2$,
    \begin{equation*}
        \Prx_{\bS \sim \mcD^n}[|\psi(\bS)| \geq \Delta] \leq 2\exp\paren*{-\frac{\Delta^2k}{2(\sigma^2 + \frac{\Delta}{3})}} \leq 2\exp\paren*{-\frac{k}{4} \min(\Delta, \Delta^2/\sigma^2)}
    \end{equation*}
    where $k \coloneqq \floor*{\frac{n}{w}}$.
\end{corollary}

Finally, we complete the proof of \Cref{thm:MI-bounds-bias}, showing how mutual information bounds bias.
\begin{proof}[Proof of \Cref{thm:MI-bounds-bias}]
    For each $y \in Y$, we define $f^{(y)}:X^n \to \R$ as
    \begin{equation*}
        f^{(y)}(S) \coloneqq \sup_{\psi \in \Psi^{(y)}} \frac{n \cdot\error(\psi, S, \mcD)}{8} - \ln (2|\Psi^{(y)}|).
    \end{equation*}
    We claim that for all $y \in Y$ and $t > 0$, $\Pr[f^{(y)}(\bS) \geq t] \leq e^{-t}$. First, consider a single test function $\psi:X^w \to [0,1]$. By \Cref{cor:bernstein} applied to the centered query $\psi' \coloneqq \psi - \psi(\mcD)$, as well as the bound $n/(2w) \leq k$ when $k \coloneqq \floor{n/w}$
    \begin{equation*}
        \Prx_{\bS \sim \mcD^n}[\error(\psi, \bS, \mcD) \cdot n/8 \geq \eps] \leq 2\exp(-\eps)
    \end{equation*}
    By the union bound,
    \begin{align*}
        \Prx_{\bS}[f^{(y)}(\bS) \geq t] &\leq 2 \sum_{\psi \in \Psi^{(y)}} \Pr\bracket*{\frac{n\cdot \error(\psi, \bS, \mcD)}{8}  - \ln (2|\Psi^{(y)}|) \geq t} \\
        & \leq 2|\Psi^{(y)}|\cdot e^{-t - \ln (2|\Psi^{(y)}|)} = e^{-t}.
    \end{align*}
    Therefore, by \Cref{cor:MI-expectation},
    \begin{equation*}
        \Ex_{\bS, \by}\bracket*{\sup_{\psi \in \Psi^{(\by)}}\frac{n\cdot \error(\psi, \bS, \mcD)}{8} - \ln (2|\Psi^{(\by)}|)} \leq 2(I(\bS; \by) + \ln 2).
    \end{equation*}
    Rearranging yields,
    \begin{equation*}
        \Ex_{\bS, \by}\bracket*{\sup_{\psi \in \Psi^{(\by)}} \error(\psi, \bS, \mcD)} \leq \frac{16I(\bS; \by) + 8\Ex_{\by}\bracket*{\ln\abs*{\Psi^{(\by)}}} +24\ln 2}{n}. \qedhere
    \end{equation*}
\end{proof}

\Cref{thm:main-binary} is a direct consequence of \Cref{thm:MI-general,thm:MI-bounds-bias}. For completeness, we show how \Cref{thm:main-binary} implies \Cref{thm:formal-simple}.

\begin{proof}[Proof of \Cref{thm:formal-simple} from \Cref{thm:main-binary}]
    For each $t \in [T]$ and $y \in Y$, we design a test query $\bpsi_{t,y}:X^w \to [0,1]$ that maps $(x_1, \ldots, x_w)$ to $\Ind[\bphi_t(x_1, \ldots, x_w) = y]$, for a total of $m = t|Y|$ test queries. Using $\boldsymbol{\Delta}_{t,y}$ and $\bsigma_{t,y}^2$ as shorthand for $\Prx_{\bT \sim \bS^{(w)}}\bracket*{\bphi_t(\bT) = y} - \Prx_{\bT \sim \mcD^w}\bracket*{\bphi_t(\bT) = y}$ and $\Var_{\bphi_{t,y}}(\mcD)$ respectively,
    \begin{align*}
        \Ex\bracket*{\sup_{t \in [T], y \in Y} \boldsymbol{\Delta}_{t,y}^2} &= \Ex\bracket*{\sup_{t \in [T], y \in Y} \min(\boldsymbol{\Delta}_{t,y}, \boldsymbol{\Delta}_{t,y}^2)} \tag{$x \geq x^2$ for $x \in [0,1]$} \\
        &\leq \Ex\bracket*{\sup_{t \in [T], y \in Y} \min(\boldsymbol{\Delta}_{t,y}, \frac{\boldsymbol{\Delta}_{t,y}^2}{\bsigma_{t,y}^2})} \tag{$\bsigma_{t,y}^2 \leq 1$}\\
        &= \Ex\bracket*{\sup_{t \in [T], y \in Y} w \cdot \error(\bpsi, \bS, \mcD)} \tag{\Cref{def:error-second}}
    \end{align*}
    The desired result follows from \Cref{thm:main-binary}.
\end{proof}

\section{Boosting the success probability}
\label{sec:reduction}

\Cref{thm:main-binary} proves that, with constant success probability, the analyst cannot find a test on which the sample is biased. In this section, we prove \Cref{thm:high-probability} showing that, when all of the analyst's queries $\phi:X^w \to Y$ satisfy $w = 1$, that guarantee holds with high probability. We do this via a reduction from small failure probability to constant failure probability.  

\pparagraph{Notation} Throughout this section, we will only consider analysts who make queries of the form $\phi:X \to Y$ and tests of the form $\psi:X \to [0,1]$ (corresponding to $w = 1$). For now, we will also restrict to analysts that only output a single test. Given an analyst $\mcA$ and sample $S \in X^n$, we'll use the notation $\mcA(S)$ as shorthand for the distribution of tests $\mcA$ asks on the sample $S$; i.e., $\bpsi \sim \mcA(S)$ is shorthand for, at each $t \in [T]$, setting $\bphi_t = \mcA(\by_1, \ldots, \by_{t-1})$ and $\by_t \sim \bphi_t(S)$, and then setting $\bpsi = \mcA(\by_1, \ldots, \by_T)$.

\begin{definition}[Budgeted analyst]
    \label{def:budgeted}
    We say an analyst $\mcA$ is $(\cost, b)$-\emph{budgeted} if, for any sample $S$, the queries $\mcA(S)$ asks, $\bphi_1, \cdots, \bphi_T$, satisfy $\sum_{t \in [T]} \cost(\bphi_t) \leq b$ almost surely.
\end{definition}

\begin{lemma}[Boosting from constant to small failure probability]
     \label{lem:auto-boost}
    For any distribution $\mcD$ over $X$, sample size $n$, cost function $\cost$, budget $b$, and threshold $\tau_{\psi} \in [0,1]$ for each $\psi:X \to [0,1]$, suppose that for all analysts $\mcA$ that are $(\cost, b)$-budgeted,
    \begin{equation*}
        \Prx_{\bS \sim \mcD^n, \bpsi \sim \mcA(\bS)}[\mu_{\bpsi}(\bS) \geq \tau_{\bpsi}] \leq \frac{1}{100}.
    \end{equation*}
    Then, for any sample size $N \geq n$, $k = \floor{N/n}$, and all analysts $\mcA'$ that are $(\cost, B \coloneqq bk/100)$-budgeted almost surely,
    \begin{equation*}
        \Prx_{bS' \sim \mcD^{N}, \bpsi \sim \mcA'(\bS')}\bracket*{\mu_{\bpsi}(\bS') \geq \tau_{\bpsi} + \frac{1}{n}} \leq \exp(-\Omega(k)).
    \end{equation*}
\end{lemma}

At a high level, the proof of \Cref{lem:auto-boost} exploits a classic and well-known technique for boosting success probabilities: Given an algorithm that fails with constant probability, if we run $k$ independent copies, with probability $1 - 2^{-\Omega(k)}$, a large fraction of those copies will succeed. Typically, this technique gives a framework for modifying existing algorithms -- for example, by taking the majority of multiple independent runs -- in order to produce a new algorithm with a small failure probability.

Interestingly, in our setting, no modification to the algorithm is necessary. If we answer subsampling queries in the most natural way, they ``automatically" boost their own success probability. The key insight, as depicted in \Cref{fig:reduction}, is that a single large sample $\bS' \sim \mcD^{N \coloneqq nk}$ can be cast as $k$ groups of samples $\bS_1, \ldots, \bS_k \iid \mcD^n$ and the output of a query $\by \sim \phi(\bS')$ is the same as that of $\by' \sim \phi(\bS_{\bi})$ where $\bi \sim [k]$ is a random group. Using this insight, we are able to prove the following.

\begin{restatable}[It is exponentially unlikely many groups have high error]{lemma}{manyGroups}
    \label{lem:many-groups}
    In the setting of \Cref{lem:auto-boost}, let $\bS_1$ be the first $n$ elements of $\bS'$, $\bS_2$ the next $n$, and so on. Then,
    \begin{equation*}
        \Prx_{\bS' \sim \mcD^N, \bpsi \sim \mcA'(\bS')}\bracket*{\sum_{i \in [k]}\Ind[\mu_{\bpsi}(\bS_i) \geq \tau_{\bpsi}] \geq 0.03k} \leq e^{-k/300}.
    \end{equation*}
\end{restatable}

The hypothesis of \Cref{lem:auto-boost} roughly corresponds to $\Pr[\mu_{\bpsi}(\bS_i) \geq \tau_{\bpsi}] \leq 1/100$ for each $i \in [k]$. If these events were independent for each $i \in [k]$, then the conclusion of \Cref{lem:many-groups} would follow from a standard Chernoff bound. Unfortunately, they are not necessarily independent. To get around that, in \Cref{lem:direct-product} we extend a direct product theorem of Shaltiel's showing, roughly speaking, those events are no worse than independent.

We combine \Cref{lem:many-groups} with the following.
\begin{restatable}{lemma}{groupsToOverall}
    \label{lem:groups-to-overall}
    For $N \geq nk$, let $\bS' \in X^N$ be drawn from any permutation invariant distribution, meaning $\Pr[\bS' = S] = \Pr[\bS' = \sigma(S)]$ for any permutation of the indices $\sigma$. Let $\bS_1$ be the first $n$ elements of $\bS'$, $\bS_2$ the next $n$, and so on. For any query $\psi:X \to Y$ and threshold $\tau$, let $\bz$ be the random variable counting the number of $i \in [k]$ for which $\mu_{\psi}(\bS_i) \geq \tau$. Then,
    \begin{equation*}
        \Pr[\mu_{\psi}(\bS') \geq \tau + 1/n]\leq 200 \Pr[\bz \geq 0.03k].
    \end{equation*}
\end{restatable}
\Cref{lem:auto-boost} is a straightforward consequence of the above two lemmas.
\begin{proof}[Proof of \Cref{lem:auto-boost} assuming \Cref{lem:many-groups,lem:groups-to-overall}]
    The analyst chooses a query $\psi^{(\by)}$ as a function of the responses $\by \coloneqq \by_1 \sim \phi_1(\bS'), \ldots, \by_T \sim \phi_T(\bS')$. Our goal is to show that $\Pr[\mu_{\psi^{(\by)}}(\bS') \geq \tau_{\psi^{(\by)}} + 1/n] \leq 200 \exp(-k/300)$.

    Conditioned on any possible sequence of responses, $\by = y$, the distribution of $\bS'$ is permutation invariant. Therefore,
    \begin{align*}
        \Pr[\mu_{\psi^{(\by)}}(\bS') \geq \tau_{\psi^{(\by)}} + 1/n] &= \Ex_{\by}\bracket*{\Prx_{\bS' \mid \by} \bracket*{\mu_{\psi^{(\by)}}(\bS') \geq \tau_{\psi^{(\by)}} + 1/n}}\\
        &\leq\Ex_{\by}\bracket*{ 200 \Prx_{\bS' \mid \by}\bracket*{\sum_{i \in [k]} \Ind[\mu_{\psi^{(\by)}}(\bS_i) \geq \tau_{\psi^{(\by)}} ] \geq 0.03k}} \tag{\Cref{lem:groups-to-overall}} \\
        &= 200 \Prx_{\bS' \sim \mcD^N, \bpsi \sim \mcA'(\bS')}\bracket*{\sum_{i \in [k]}\Ind[\mu_{\psi}(\bS_i) \geq \tau_{\bpsi}] \geq 0.03k} \\
        &\leq 200 e^{-k/300}\tag{\Cref{lem:many-groups}}.
    \end{align*}
\end{proof}

\subsection{Proof of \texorpdfstring{\Cref{lem:many-groups}}{Lemma \ref{lem:many-groups}}}

In order to prove \Cref{lem:many-groups}, in \Cref{fig:analyst-game} we will define an ``analyst game" formalizing the setting in which an analyst has multiple distinct samples each of which they can ask queries to. We then prove a direct product theorem analogous to to Shaltiel's direct product theorem for fair decision trees \cite{Sha04}.
\begin{figure}[ht] 

  \captionsetup{width=.9\linewidth}

    \begin{tcolorbox}[colback = white,arc=1mm, boxrule=0.25mm]
    \vspace{2pt} 
    \textbf{Parameters:}  A product distribution $\mcD \coloneqq \mcD_1 \times \cdots \times \mcD_k$ over domain $X$, budgets $b \in \R^k$, a class of queries $\Phi$ each mapping $X$ to a distribution of possible responses, function $\cost: \Phi \to \R_{\geq 0}$, and class of tests $\Psi$ each mapping $X$ to $\{0,1\}$. \\

    \textbf{Setup:} Samples $\bx_1, \ldots, \bx_k \sim \mcD$ are drawn and \emph{not} revealed to the analyst.\\

    \textbf{Execution:} The analyst repeats as many times as desired:
    \begin{enumerate}[nolistsep,itemsep=2pt]
        \item The analyst chooses a query $\phi \in \Phi$ and index $i \in [k]$
        \item The analyst receives the response $\by \sim \phi(\bx_i)$.
        \item The budget is decremented: $b_i \leftarrow b_i - \cost(\phi)$.
    \end{enumerate}\vspace{4pt}
    Afterwards, the analyst chooses tests $\psi_1, \ldots, \psi_k$. The analyst wins if $b_i \geq 0$ and $\psi_i(\bx_i) = 1$ for all $i \in [k]$.
    \end{tcolorbox}
\caption{An analyst game.}
\label{fig:analyst-game}
\end{figure} 

\begin{lemma}[Direct product theorem for analyst games]
    \label{lem:direct-product}
    In \Cref{fig:analyst-game}, fix the domain $X$, query class $\Phi$, test class $\Psi$ and cost function $\cost$ for which $\inf_{\phi \in \Phi}(\cost(\phi)) > 0$. For any distributions $\mcD \coloneqq \mcD_1 \times \cdots \mcD_k$ each supported on $X$ and budgets $b \in \R^k$, let $\AG(\mcD, b)$ be the maximum probability an analyst wins the game described in \Cref{fig:analyst-game}. Then,
    \begin{equation}
        \label{eq:direct-prod}
        \AG(\mcD, b) \leq \prod_{i \in [k]} \AG(\mcD_i, b_i).
    \end{equation}
\end{lemma}

It's straightforward to see that the $(\geq)$ direction of \Cref{eq:direct-prod} also holds, but we will not need that direction in this paper.
\begin{proof}
    First note that if, for any $i \in [k]$, $b_i < 0$, then both sides of \Cref{eq:direct-prod} are equal to $0$. We may therefore assume, without loss of generality that $b_i \geq 0$ for all $i \in [k]$.

    Consider an arbitrary analyst, $\mcA$, for distribution $\mcD$ and budget $b$. We will prove that the probability $\mcA$ wins is at most $\prod_{i \in [k]} \AG(\mcD_i, b_i)$ by induction on the number of queries the analyst makes.\footnote{Note that since $\inf_{\phi \in \Phi}(\cost(\phi)) > 0$, the number of queries the analyst makes must be finite. Strictly speaking, we must also assume that the analysts stops making queries once one of the budgets is negative, which we can clearly assume without loss of generality.} In the base case, $\mcA$ makes zero queries and directly chooses tests. Then, $\mcA$'s probability of winning is upper bounded by
    \begin{align*}
        \sup_{\phi_1, \ldots, \phi_k \in \Psi} \Prx_{\bx \sim \mcD}\bracket*{ \prod_{i \in [k]} \psi_i(\bx_i)} &= \sup_{\phi_1, \ldots, \phi_k \in \Psi}  \prod_{i \in [k]}\Prx_{\bx \sim \mcD}\bracket*{  \psi_i(\bx_i)} \tag{$\mcD$ is product} \\
        &=  \prod_{i \in [k]}\sup_{\phi_i \in \Psi} \Prx_{\bx_i \sim \mcD_i}\bracket*{  \psi_i(\bx_i)} \leq \AG(\mcD_i, b_i).
    \end{align*}
    In the case where $\mcA$ executes at least one query, let $\phi \in \Phi$ and $i\in[k]$ be the query and group respectively that $\mcA$ chooses first. Using $b' \in \R^k$ as the vector satisfying $b'_i = b_i - \cost(\phi)$ and $b'_j = b_j$ for all $j \neq i$, the success probability of $\mcA$ is upper bounded by
    \begin{align*}
        \Ex_{\bx \sim \mcD, \by \sim \phi(\bx_i)} &\bracket*{\AG(\mcD \mid \phi(\bx_i) = \by, b')} \\
        &\leq \Ex_{\bx \sim \mcD, \by \sim \phi(\bx_i)} \bracket*{\prod_{j \in [k]}\AG((\mcD \mid \phi(\bx_i) = \by)_j, b'_j)} \tag{inductive hypothesis} \\
        & =\Ex_{\bx \sim \mcD, \by \sim \phi(\bx_i)} \bracket*{\AG((\mcD \mid \phi(\bx_i) = \by)_i, b_i - \cost(\phi)) \cdot \prod_{j \neq i} \AG(\mcD_j, b_j)}\\
        &= \prod_{j \neq i} \AG(\mcD_j, b_j) \cdot \Ex_{\bx \sim \mcD, \by \sim \phi(\bx_i)} \bracket*{\AG((\mcD \mid \phi(\bx_i) = \by)_i, b_i - \cost(\phi)}.
    \end{align*}
    The quantity $\Ex_{\bx \sim \mcD, \by \sim \phi(\bx_i)} \bracket*{\AG((\mcD \mid \phi(\bx_i) = \by)_i, b_i - \cost(\phi)}$ is exactly the win probability on $\mcD_i, b_i$ of the analyst whose first query is $\phi$, and remaining strategy is optimal. Therefore, it is upper bounded by $\AG(\mcD_i, b_i)$ and so the win probability of $\mcA$ is upper bounded by $\prod_{i \in [k]} \AG(\mcD_i, b_i)$, as desired.
\end{proof}

The most useful form of \Cref{lem:direct-product} in this work will give a bound on the number of groups for which the tests pass. That version, given in \Cref{cor:direct-product-count}, uses the following Chernoff bound.
\begin{fact}[Chernoff bound]
    \label{fact:chernoff}
    Let $\bx_1, \ldots, \bx_k$ be random variables each taking on values in $\zo$ and satisfying, for some $p < 1$ and all $S \subseteq [k]$,
    \begin{equation*}
        \Ex\bracket*{\prod_{i \in S}\bx_i} \leq p^{|S|}.
    \end{equation*}
    Then, for any $\delta > 0$,
    \begin{equation*}
        \Pr\bracket*{\sum_{i \in [k]} \bx_i \geq (1 +\delta)pk} \leq \exp\paren*{-\frac{\delta^2 pk}{ 2+\delta}}.
    \end{equation*}
\end{fact}

In the below corollary, for notational convenience, we will only consider the case where all the groups are identical, meaning $\mcD_i = \mcD_j$ and $b_i = b_j$ for all $i \neq j$. However, a similar result would hold even if the groups were different.
\begin{corollary}
    \label{cor:direct-product-count}
    Consider a modified analyst game where rather than the analyst only winning if $b_i \geq 0$ and $\psi_i(\bx_i) = 1$ for all $i \in [k]$, we instead give the analyst a ``score" based on the number of such groups there are,
    \begin{equation*}
        \textbf{score} = \sum_{i \in [k]} \Ind[b_i \geq 0 \text{ and }\psi_i(\bx_i) = 1].
    \end{equation*}
    Then for product distribution $\mcD = (\mcD_{\base})^k$, identical budgets for each coordinate, $b_i = b_{\base}$ for every $i \in [k]$, $p \coloneqq \AG(\mcD_{\base}, b_{\base})$, and any strategy of the analyst,
    \begin{equation*}
        \Pr[\textbf{score} \geq 2kp] \leq \exp\paren*{-\frac{pk}{3}}.
    \end{equation*}
\end{corollary}
\begin{proof}
    By \Cref{fact:chernoff}, it suffices to show that for any $S \subseteq [k]$,
    \begin{equation*}
        \Ex\bracket*{\prod_{i \in S} \Ind[b_i \geq 0 \text{ and }\psi_i(\bx_i) = 1]} \leq p^{|S|}.
    \end{equation*}
    We will prove the above holds for any such $S$ by applying \Cref{lem:direct-product}. We originally have an analyst, $A$, operating over all $k$ groups. To apply \Cref{lem:direct-product}, we will create an analyst, $A_S$, operating over just the groups in $S$. The analyst $A_S$ proceeds as follows.
    \begin{enumerate}
        \item At the beginning, for each $i \notin S$, the analyst draws $\bx_i \iid \mcD_{\base}$.
        \item For each query $A$ makes, if the query is to some $i \in S$, $A_S$ makes the same query. Otherwise, $A_S$ simulates the query by drawing $\by \sim \phi(\bx_i)$. Note that for such $i \notin S$, the analyst knows $\bx_i$ and so it does not pay for such a simulated query.
        \item After all queries are made, $A_S$ chooses the same test as $A$ for every $i \in S$ and ignores the tests made to $i \notin S$.
    \end{enumerate}
    Then, the probability that $b_i \geq 0 \text{ and }\psi_i(\bx_i) = 1$ for all $i \in S$ is the same for analysts $A$ and $A_S$. The desired result then follows from \Cref{lem:direct-product}.
\end{proof}

We now prove the main result of this subsection, restated for convenience.
\manyGroups*
\begin{proof}
    Let $\bS_{k+1}$ be the remaining $N - nk$ elements not in $\bS_{1}, \ldots, \bS_{k}$. At each time step $t \in [T]$, the analyst asks a query $\phi_t:X \to Y_t$ and receives the response $\by_t = \phi_t(\bx_t)$ for $\bx_t$ uniform from $\bS'$. We can think of $\bx_t$ as being chosen via a two step process: First, a group $\bi_t \in [k+1]$ is chosen (with $\Pr[\bi_t = i] = |\bS_i|/N$), and second $\bx_t$ is chosen uniformly from $\bS_{\bi}$. We will show that \Cref{lem:many-groups} holds \emph{even} conditioned on any choice of $\bi_1 = i_1, \ldots, \bi_{T} = i_T$.

     Since there is a total of $B = bk/100$ budget, regardless of how the analyst partitions the queries to the groups, at most $0.01k$ groups will receive queries with total cost $> b$. Furthermore, by \Cref{cor:direct-product-count} combined with the hypothesis of \Cref{lem:auto-boost}, the probability that there are at least $0.02k$ groups $i \in [k]$ that both receive queries with a budget of at most $b$ and satisfy $\mu_{\bphi}(\bS_i) \geq \tau_{\bpsi}$ is $e^{-k/300}$. Combining these gives the desired bound.
\end{proof}

\subsection{Proof of \texorpdfstring{\Cref{lem:groups-to-overall}}{Lemma \ref{lem:groups-to-overall}}}

We begin my proving the following technical lemma.
\begin{lemma}
    \label{lem:sample-exceeds-mean}
    For any $S \in [0,1]^N$ and $n < N$, let $\bx$ be sampled by taking the sum of $n$ elements from $S$ chosen uniformly without replacement. Then,
    \begin{equation*}
        \Pr[\bx > \Ex[\bx] - 1] \geq \frac{2 \sqrt{3} - 3}{13} > 0.0357.
    \end{equation*}
\end{lemma}
We conjecture that the probability $\bx$ exceeds $\Ex[\bx]-1$ is actually lower bounded by $1/2$ - such a result is known if, for example, $\bx$ were drawn from a binomial distribution \cite{N66,KB80} or hypergeometric distribution \cite{siegel2001median}. For our purposes, any constant probability suffices and we are therefore able to use a quantitatively weaker bound that only uses two moments of $\bx$.
\begin{fact}[\cite{V08}]
    \label{fact:moment-to-exceed-mean}
    For any mean-$0$ random variable $\bx$ with nonzero variance,
    \begin{equation*}
        \Pr[\bx > 0] \geq \frac{2\sqrt{3} - 3}{\frac{\Ex[\bx^4]}{\Ex[\bx^2]^2}}.
    \end{equation*}
\end{fact}
By Jensen's inequality, $\frac{\Ex[\bx^4]}{\Ex[\bx^2]^2}$ is at least $1$. For that ratio to be small, $\bx$ must concentrate well. The intuition behind \Cref{fact:moment-to-exceed-mean} is that for $\bx$ to rarely exceed its mean, it must be strongly skewed to one side, which will lead to the ratio of moments being large. Next, we bound this ratio for the particular $\bx$ we care about. To make the arithmetic simpler, we compare to the setting where the sample was performed \emph{with replacement} using Hoeffding's reduction theorem (\Cref{thm:hoef-63-intro}).

\begin{proposition}
\label{prop:reasonable-replacement}
    For any $S \in \R^N$ with elements summing to $0$ and $n \leq N$. Let $\bx$ be the sum of $n$ elements sampled uniformly \emph{without replacement} from $S$, and $\by$ be the sum of $n$ elements sampled uniformly \emph{with replacement} from $S$. Then,
    \begin{equation*}
        \frac{\Ex[\bx^4]}{\Ex[\bx^2]^2} \leq \paren*{\frac{N-1}{N-n}}^2 \cdot  \frac{\Ex[\by^4]}{\Ex[\by^2]^2}.
    \end{equation*}
\end{proposition}
\begin{proof}
    By Hoeffding's reduction theorem (\Cref{thm:hoef-63-intro}), $\Ex[\bx^4] \leq \Ex[\by^4]$. Therefore  it suffices to show that $\Ex[\bx^2]= \Ex[\by^2] \cdot  \frac{N-n}{N-1}$.

    Let $\bx_1, \ldots, \bx_n$ and $\by_1, \ldots, \by_n$ be the $n$ elements of $S$ chosen without replacement and with replacement respectively. Then, since the $\by_i$ are mean-$0$ and independent, for $i \neq j \in [n]$, $\Ex[\by_i \by_j] = 0$. Meanwhile, using the fact that the sum of elements in $S$ is $0$,
    \begin{equation*}
        \Ex[\bx_i \bx_j] = \frac{1}{N(N-1)}\sum_{a \in [N]}\sum_{b \in [N] \setminus \set{a}}S_a S_b = \frac{1}{N(N-1)}\sum_{a \in [N]} S_a \paren*{\sum_{b \in [N]}S_b - S_a} = -\frac{1}{N(N-1)} \sum_{a \in [N]}S_a^2.
    \end{equation*}
    Furthermore, we have for any $i \in [n]$, that
    \begin{equation*}
        \Ex[\bx_i^2] = \Ex[\by_i^2] = \frac{1}{N} \sum_{a \in [N]} S_a^2.
    \end{equation*}
    Next, we can compute the second moment of $\by$,
    \begin{equation*}
        \Ex[\by^2] = \sum_{i \in [n]}\Ex[\by_i^2] = \frac{n}{N}\sum_{a \in [N]} S_a^2.
    \end{equation*}
    For $\bx$,
    \begin{equation*}
        \Ex[\bx^2] = \sum_{i \in [n]}\Ex[\bx_i^2] + \sum_{i \neq j} \Ex[\bx_i \bx_j] = \frac{n}{N}\sum_{a \in [N]} S_a^2 - \frac{n(n-1)}{N(N-1)}\sum_{a \in [N]} S_a^2 = \frac{n}{N} \paren*{1 - \frac{n-1}{N-1}}\sum_{a \in [N]} S_a^2.
    \end{equation*}
    Comparing the above gives the desired bound.
\end{proof}

Next, we bound the desired moment ratio in the setting where the elements are sampled \emph{with} replacement.

\begin{proposition}
    \label{prop:moment-ratio}
    Let $\by_1, \ldots, \by_n$ be iid mean-$0$ variables each bounded on $[-1,1]$. Then, for $\by \coloneqq \sum_{i \in [n]}\by_i$
    \begin{equation*}
        \frac{\Ex[\by^4]}{\Ex[\by^2]^2} \leq 3 + \frac{1}{\Var[\by]}
    \end{equation*}
\end{proposition}
\begin{proof}
    We'll denote $\sigma^2 \coloneqq \Ex[\by_i^2]$. Since $\by_i$ is bounded, we further have that $\Ex[\by_i^4] \leq \sigma^2$. Then,
    \begin{equation*}
        \Ex[\by^2] = n \sigma^2.
    \end{equation*}
    Expanding $\Ex[\by^4]$ gives many terms. Luckily, since the $\by_i$ are each independent and mean-$0$, most of those terms cancel. We are left with,
    \begin{equation*}
        \Ex[\by^4] = n\Ex[\by_i^4] + 3n(n-1)(\sigma^2)^2 \leq n\sigma^2 + 3n^2\sigma^4.
    \end{equation*}
    This gives
    \begin{equation*}
        \frac{\Ex[\by^4]}{\Ex[\by^2]^2} \leq \frac{n\sigma^2 + 3n^2 \sigma^4}{n^2\sigma^4} = 3 + \frac{1}{n\sigma^2}.
    \end{equation*}
    The desired result follows from $\Var[\by] = n\sigma^2$.
\end{proof}
By combining the above, we have the following.
\begin{restatable}{lemma}{exceedMean}
    \label{lem:sample-exceeds-mean-var}
    For any $S \in [0,1]^N$ and $n < N$, let $\bx$ be sampled by taking the sum of $n$ elements from $S$ chosen uniformly without replacement. Then, if $\Var[\bx]$ is nonzero,
    \begin{equation*}
        \Pr[\bx > \Ex[\bx]] \geq \frac{2 \sqrt{3} - 3}{12 + \frac{4}{\Var[\bx]}}
    \end{equation*} 
\end{restatable}

\begin{proof}
    Let $\mu \coloneqq \frac{1}{N} \sum_{x \in S} x$, and $S'$ be a copy of $S$ with each element shifted by $-\mu$. Clearly, each element of $S'$ is bounded on $[-1,1]$ and they sum to $0$, so it suffices to bound $\Pr[\bx' > 0]$ for $\bx'$ being the sum of $n$ uniform elements chosen without replacement from $S'$. Furthermore, without loss of generality, we may assume that $n \leq N/2$ as, if $n > N/2$, we may instead consider the sum of the elements \emph{not} sampled; $\bx' > 0$ iff the sum of the elements not sampled is negative.

    Let $\by'$ be the sum of $n$ elements chosen uniformly with replacement from $S'$. Then,
    \begin{align*}
        \Pr[\bx > \Ex[\bx]] &= \Pr[\bx' > 0] \\
        &\geq \frac{2\sqrt{3} - 3}{\frac{\Ex[\bx'^4]}{\Ex[\bx'^2]^2}} \tag{\Cref{fact:moment-to-exceed-mean}} \\
        & \geq \frac{2\sqrt{3} - 3}{\paren*{\frac{N-1}{N-n}}^2\frac{\Ex[\by'^4]}{\Ex[\by'^2]^2}} \tag{\Cref{prop:reasonable-replacement}}\\
        & \geq \frac{2\sqrt{3} - 3}{4\frac{\Ex[\by'^4]}{\Ex[\by'^2]^2}} \tag{$n \leq N/2$}\\
        & \geq \frac{2\sqrt{3} - 3}{4\paren*{3 + \frac{1}{\Var[\by']}}} \tag{\Cref{prop:moment-ratio}}\\
        & \geq \frac{2\sqrt{3} - 3}{4\paren*{3 + \frac{1}{\Var[\bx']}}} \tag{\Cref{thm:hoef-63-intro}}\\
        & = \frac{2\sqrt{3} - 3}{4\paren*{3 + \frac{1}{\Var[\bx]}}} \tag{$\bx'$ is shifted version of $\bx$}
    \end{align*}
\end{proof}
\Cref{lem:sample-exceeds-mean} handles the case where the variance is large. For the case where it is small, we use the following one-sided variant of Chebyshev's inequality.
\begin{fact}[Cantelli's inequality]
    \label{fact:cantelli}
    Let $\bx$ be a random variable with variance $\sigma^2$. Then,
    \begin{equation*}
        \Pr[\bx - \Ex[\bx] \geq \eps \sigma] \leq \frac{1}{1 + \eps}.
    \end{equation*}
\end{fact}
Finally, we prove \Cref{lem:sample-exceeds-mean}.
\begin{proof}[Proof of \Cref{lem:sample-exceeds-mean}]
    If $\Var[\bx] \geq 4$, the desired result follows from \Cref{lem:sample-exceeds-mean-var}. Otherwise, it follows from \Cref{fact:cantelli}.
\end{proof}

We conclude with a proof of the main result of this section, restated for convenience.
\groupsToOverall*

\begin{proof}[Proof of \Cref{lem:groups-to-overall}]
    We'll use $\bE$ as shorthand for $\Ind[\mu_{\psi}(\bS') \geq \tau + 1/n]$. Since $\Pr[\bz \geq 0.03k] \geq \Pr[\bz \geq 0.03k \mid \bE] \cdot \Pr[\bE]$,
    \begin{equation*}
         \Pr[\bE] \leq \frac{\Pr[\bz \geq 0.03k]}{\Pr[\bz \geq 0.03k \mid \bE]}.
    \end{equation*}
    Therefore, it suffices to show that $\Pr[\bz \geq 0.03k \mid \bE] \geq \frac{1}{200}$. By \Cref{lem:sample-exceeds-mean}, for any $i \in [k]$
   \begin{equation*}
       \Pr[\mu_{\psi}(\bS_i) \geq \tau \mid \bE] > 0.0357.
   \end{equation*}
    Using linearity of expectation,
    \begin{equation*}
        \Ex[\bz \mid \bE] > 0.0357k.
    \end{equation*}
    The random variable $k - \bz$ is nonnegative and satisfies
    \begin{equation*}
        \Ex[k - \bz \mid \bE] < k -0.0357k = 0.9643k.
    \end{equation*}
    Therefore, by Markov's inequality
    \begin{equation*}
        \Pr[k - \bz \geq 0.97k \mid \bE] \leq \frac{0.9643k}{0.97k} < 0.995.
    \end{equation*}
    Equivalently, $\bz \geq 0.03k$ with probability at least $0.005$ conditioned on $\mu_{\psi}(\bS') \geq\tau + 1/n$, which is exactly what we wished to show.
\end{proof}

\subsection{Proof of \texorpdfstring{\Cref{thm:high-probability}}{Theorem \ref{thm:high-probability}}}
\begin{proof}
    The case where $m \geq 2$ follows from the $m = 1$ case and a union bound, so we focus on the $m=1$ case.  Let $k = O(\ln(1/\delta))$. We note that the desired bound is vacuous if $k > n$, so we may assume that $k \leq n$. Therefore, $n' \coloneqq \floor{n/k}$ satisfies $n' \geq n/{2k}$. By \Cref{thm:main-binary}, for any analyst $\mcA'$ that is $(\cost, 100b/k)$-budgeted
    \begin{equation*}
        \Ex_{\bS \sim \mcD^{n'}, \bpsi \sim \mcA'(\bS)}[\error(\bpsi, \bS, \mcD)] \leq O\paren*{\frac{b}{k(n')^2} + \frac{1}{n'}} = O\paren*{k\cdot \paren*{\frac{b}{n^2} + \frac{1}{n}}}
    \end{equation*}
    For each possible test function $\psi:X \to [0,1]$, we set the threshold to
    \begin{equation*}
        \tau_{\psi} \coloneqq \mu_{\phi}(\mcD) + O\paren*{\max\paren*{\frac{kb}{n^2} + \frac{k}{n}, \sqrt{\paren*{\frac{kb}{n^2} + \frac{k}{n} }\cdot \Var_{\psi}(\mcD)}}},
    \end{equation*}
    so that, by Markov's inequality and the expected error bound,
    \begin{equation*}
        \Prx_{\bS \sim \mcD^{n'}, \bpsi \sim \mcA'(\bS)}[\mu_{\bpsi}(\bS) \geq \tau_{\bpsi}] \leq \frac{1}{100}.
    \end{equation*}
    Here, we apply \Cref{lem:auto-boost}.
    \begin{equation*}
        \Prx_{\bS \sim \mcD^{n}, \bpsi \sim \mcA(\bS)}\bracket*{\mu_{\bpsi}(\bS) \geq \tau_{\psi} + \frac{1}{n'}} \leq  \Prx_{\bS \sim \mcD^{n}, \bpsi \sim \mcA(\bS)}\bracket*{\mu_{\bpsi}(\bS) \geq \tau_{\psi} + \frac{2k}{n}} \leq \exp(-\Omega(k)) = \frac{\delta}{2}.
    \end{equation*}
    We can symmetrically bound the probability that $\mu_{1 -\bphi}$ is too high, which is equivalent to bounding the probability $\mu_{\bphi}$ is too low. Union bounding over these two cases, we have that,
    \begin{equation*}
         \Prx_{\bS \sim \mcD^{n}, \bpsi \sim \mcA(\bS)}\bracket*{\error(\bpsi, \bS,\mcD) \geq  O\paren*{k\cdot \paren*{\frac{b}{n^2} + \frac{1}{n}}}}  \leq \delta. \qedhere
    \end{equation*}

\end{proof}

\section{Applications}
\label{sec:apps}

\subsection{The statistical-query mechanism: Proof of \texorpdfstring{\Cref{thm:SQ}}{Theorem 6}}

The proof of \Cref{thm:SQ} combines two bounds: First, with high probability, all statistical queries asked have low bias, meaning $\mu_{\phi}(\bS)$ and $\mu_{\phi}(\mcD)$ are close. Second, with high probability, the answer the mechanism gives when receiving a statistical query $\phi$ is close to $\mu_{\phi}(\bS)$. These two results are combined with the following proposition.

\begin{proposition}[Triangle-like inequality]
    \label{prop:triangle-like}
    For any $\tau > 0$ and $a,b,c \in [0,1]$ satisfying
    \begin{equation*}
        \abs{b - a} \leq \max\paren*{\tau^2, \tau\sqrt{a(1-a)}}\quad\quad\quad\quad\abs{c - b} \leq \max\paren*{\tau^2, \tau\sqrt{b(1-b)}}
    \end{equation*}
    it also holds that
    \begin{equation*}
        \abs{c - a} \leq 3\max\paren*{\tau^2, \tau\sqrt{a(1-a)}}. 
    \end{equation*}
\end{proposition}
\begin{proof}
    First, consider the case where $\abs{c - b}\leq \tau^2$. Then,
    \begin{equation*}
        \abs{c-a} \leq \abs{a-b} + \abs{c-b} \leq  \max\paren*{\tau^2, \tau\sqrt{a(1-a)}} + \tau^2 \leq 2\max\paren*{\tau^2, \tau\sqrt{a(1-a)}}.
    \end{equation*}
    In the other case,
    \begin{align*}
         \abs{c-b} &\leq \tau\sqrt{b(1-b)} \\
         &= \tau \sqrt{(a + (b-a))(1 - a - (b-a))} \\
         &= \tau\sqrt{a(1-a) + (b-a)(1-2a) - (b-a)^2}\\
         &\leq \tau\sqrt{a(1-a) + |b-a|} \\
         &\leq \tau\sqrt{a(1-a)} + \tau\sqrt{|b-a|} \\
         &\leq \tau\sqrt{a(1-a)} + \frac{\tau^2 + |b-a|}{2} \tag{$2xy \leq x^2 + y^2$ for any $x,y \in \R$}\\
         &\leq \tau\sqrt{a(1-a)} + \frac{2\max(\tau^2, \tau \sqrt{a(1-a)})}{2} \tag{$|b-a| \leq \max(\tau^2, \tau \sqrt{a(1-a)})$}\\
         &\leq 2\max(\tau^2, \tau \sqrt{a(1-a)}). \qedhere
    \end{align*}
   The desired result follows from the standard triangle inequality.
\end{proof}

The difference between $\by_t$ and $\phi_t(\bS)$ is bounded by the following form of Chernoff bounds.
\begin{fact}[Chernoff bounds]
    \label{fact:chernoff-applications}
    Let $\by$ be the mean of $k$ independent random variables each supported on $\zo$. If
    \begin{equation*}
        k = O\paren*{\ln(1/\delta)/\tau^2},
    \end{equation*}
    then, with probability at least $1 - \delta$,
    \begin{equation*}
        \abs{\by - p} \leq \max\paren*{\tau^2, \tau \sqrt{p(1-p)}}
    \end{equation*}
    where $p \coloneqq \Ex[\by]$
\end{fact}
Applying the above and a union bound over the $T$ queries, we arrive at the following.
\begin{corollary}
    \label{cor:SQ-all-close}
    In the setting of \Cref{thm:SQ}, with probability at least $1 - \delta$, for all $t \in [T]$,
    \begin{equation*}
        \abs{\phi_t(\bS) - \by_t} \leq \max\paren*{\tau^2, \tau \sqrt{\phi_t(\bS)(1 - \phi_t(\bS))}}.
    \end{equation*}
\end{corollary}

We are ready to prove that our SQ mechanism is accurate with high probability. This proof will apply the ``monitor" technique of Bassily, Nissim, Smith, Steinke, Stemmer, and Ullman \cite{BNSSSU16} to set the test query to the SQ with the ``worst" response. To select this worst query, the analyst must have some knowledge of the distribution as they must measure the deviation between $\by_t$ and $\mu_{\phi_t}(\mcD)$.

At first glance, this is strange: If the analyst knew $\mcD$, they would already know $\mu_{\phi_t}(\mcD)$ to perfect accuracy and therefore have no need to use our mechanism. The key is that, while we assume the analyst knows $\mcD$ as part of the analysis, we use this to conclude that the mechanism's responses, $\by_1, \ldots, \by_T$, have good accuracy. Therefore, even if the analyst did not know $\mcD$, they could use the mechanism's responses to estimate the answer to their queries. We'll also mention that, from the perspective of applying \Cref{thm:high-probability}, we are free to assume the analyst knows $\mcD$ as the guarantee of \Cref{thm:high-probability} holds for all distributions and analysts.

\begin{proof}[Proof of \Cref{thm:SQ}]
    After asking the adaptive sequence of queries $\phi_1, \ldots, \phi_T$ and receiving responses $\by_1, \ldots, \by_T$, the analyst sets the test query to
    \begin{equation*}
        \psi \coloneqq \phi_{t^\star} \quad\text{where}\quad t^\star \coloneqq \argmax_{t \in [T]} \frac{\abs{\by_t - \mu_{\phi_t}(\mcD)}}{\max\paren*{\tau^2, \tau \std_{\phi_t}(\mcD)}}.
    \end{equation*} 
    It is sufficient to show that \Cref{eq:SQ-accuracy} holds for $t = t^\star$ as, based on how we defined $t^\star$, it then holds for all $t \in [T]$.

    In order to execute the mechanism in \Cref{fig:SQ-mechanism}, the analyst makes a total of $Tk$ subsampling queries, one for each vote generated, each with $w = 1$ and a range $|Y| = |\zo| = 2$. As a consequence of \Cref{thm:high-probability}, with probability at least $1 - \delta$,
    \begin{equation*}
        \error(\phi_{t^\star}, \bS, \mcD) \leq O\paren*{\ln(1/\delta) \cdot \paren*{\frac{Tk}{n^2} + \frac{1}{n}}} = O\paren*{\ln(1/\delta) \cdot \paren*{\frac{\tau^2}{\ln(1/\delta)}}} = \tau^2.
    \end{equation*}
    The above implies that
    \begin{equation*}
        \abs{\mu_{\phi_{t^\star}}(\bS) - \mu_{\phi_{t^\star}}(\mcD)} \leq \max\paren*{\tau^2, \tau \sqrt{\Var_{\phi_{t^\star}}(\mcD)}} \leq \max\paren*{\tau^2, \tau \std_{\phi_{t^\star}(\mcD)}}.
    \end{equation*}
    We furthermore know by \Cref{cor:SQ-all-close} that, with probability at least $1 - \delta$,
    \begin{equation*}
         \abs{\mu_{\phi_{t^\star}}(\bS) - \by_{t^\star}} \leq\max\paren*{\tau^2, \tau \sqrt{\mu_{\phi_{t^\star}}(\bS)(1-\mu_{\phi_{t^\star}}(\bS))}}.
    \end{equation*}
      Note that it is important that \Cref{cor:SQ-all-close} union bounded over all of $t = 1, \ldots, T$ because it is therefore guaranteed to apply for $t^\star$. We couldn't have simply bounded the above for $t^\star$ by directly applying a Chernoff bound to it, as the selection process for $t^\star$ depends on the values of $\by_1, \ldots, \by_T$ and is more likely to choose a $\by_t$ that deviates greatly from its mean.
    
    Finally, union bounding both events that each occur with probability $1-\delta$ and applying \Cref{prop:triangle-like}, we have that with probability at least $1-2\delta$,
    \begin{equation*}
        \abs{\mu_{\phi_{t^\star}}(\mcD) - \by_{t^\star}} \leq 3\max\paren*{\tau^2, \tau \std_{\phi_{t^\star}(\mcD)}}.
    \end{equation*}
    The desired result follows by renaming $\delta' \coloneqq \delta/2$ and $\tau' \coloneqq \tau/3$.    
\end{proof}

\subsection{The median-finding mechanism: Proof of \texorpdfstring{\Cref{thm:median}}{Theorem \ref{thm:median}}}

In this section, we prove \Cref{thm:median}. The first step will be to quantify when a sample $S$ is representative of the distribution $\mcD$ for one particular query $\phi$.
\begin{definition}[Good and bad samples for a query]
    \label{def:good-group}
    For a query $\phi:X^w \to R \subseteq \R$ and distribution $\mcD$ over $X$, we say that a sample $S \in X^n$ is \emph{good} for query $\phi$ if
    \begin{equation*}
        \Prx_{\bz \sim \phi(S)}[\bz < r_{\min}] \leq 0.45 \quad\quad\text{and}\quad\quad\Prx_{\bz \sim \phi(S)}[\bz > r_{\max}] \leq 0.45,
    \end{equation*}
    where $r_{\min}$ and $r_{\max}$ are the smallest and largest approximate medians of $\phi(\mcD)$ respectively. If $S$ is not \emph{good} for query $\phi$, we refer to it as \emph{bad} for query $\phi$.
\end{definition}

The proof of \Cref{thm:median} is broken into two steps. First, we show that if, for all $t \in [T]$ only a small fraction ($\leq 0.02k)$ of the groups $S_1, \ldots, S_k$ are bad for $\phi$, then the mechanism will succeed with high probability. Second, we show that with high probability, only a small fraction of the groups are bad.

\begin{claim}[If most groups are good, the mechanism is correct with high probability]
    \label{claim:median-good-means-accurate}
    For any distribution $\mcD$, groups $S_1, \ldots, S_k$, and query $\phi:X^w \to R$, let $\by$ be the value output by the mechanism in \Cref{fig:median-mechanism}. If at most $0.02k$ of the groups $S_1, \ldots, S_k$ are bad for $\phi$, then
    \begin{equation*}
        \Pr[\by\text{ is not an approximate median of } \phi(\mcD)] \leq \log_2 |R| \cdot e^{-\Omega(k)}.
    \end{equation*}
\end{claim}
\begin{proof}
    Let $r_{\min}$ and $r_{\max}$ be the smallest and largest approximate medians of $\phi(\mcD)$ respectively. For $\by$ to not be an approximate median, there must have been some threshold $r$ tested during the binary search in which one of the following occurred.
    \begin{enumerate}
        \item The mechanism determined that $\by \geq r$ when $r > r_{\max}$.
        \item The mechanism determined that $\by < r$ when $r \leq r_{\min}$.
    \end{enumerate}
    The binary search only tests $\log_2 |R|$ values of $r$, so it suffices to show that for each such test, the probability one of the above two cases happens is at most $e^{-\Omega(k)}$. 
    
    Fix any choice of $r$. If $r > r_{\max}$, we wish to argue that it is unlikely the mechanism determines that $\by \geq r$, or equivalently, it is unlikely that at least $k/2$ of the votes $\bv_1, \ldots, \bv_k$ are set to $1$. If $i \in [k]$ is a good group, using \Cref{def:good-group},
    \begin{equation*}
        \Pr[\bv_i = 1] = \Prx_{\bz \sim \phi(S_i)}[\bz \geq r] \leq  \Prx_{\bz \sim \phi(S_i)}[\bz > r_{\max}] \leq 0.45.
    \end{equation*}
    For the bad groups, we can pessimistically bound $\Pr[\bv_i = 1] \leq 1$. Then, using the assumption that at most $0.02k$ groups are bad, we have that
    \begin{equation*}
        \Ex\bracket*{\sum_{i \in [k]}\bv_i} \leq 0.98k \cdot 0.45 + 0.02k \cdot 1 \leq 0.47k.
    \end{equation*}
    Using the independence of $\bv_1, \ldots, \bv_k$ and a Chernoff bound, the probability at least $k/2$ votes are $1$ is $e^{-\Omega(k)}$, as desired.

    The other case, in which $r \leq r_{\min}$, is symmetric. Here, we need to argue it is unlikely the mechanism determines that $\by < r$, which means it is unlikely that at least $k/2$ votes are $0$. Using the same logic as in the prior case, the expected number of $0$ votes is at most $0.47k$ and the same Chernoff bound applies.
\end{proof}

\begin{claim}
    \label{claim:likely-most-groups-good}
    For $\bS \sim \mcD^n$ be the input to the mechanism in \Cref{fig:median-mechanism}. Then, using the parameters set in \Cref{thm:median}, for any query $t \in [T]$, the probability that more than $0.02k$ of the groups $\bS_1, \ldots, \bS_k$ are bad for $\bphi_t$ is at most $e^{-k/300}$.
\end{claim}
To prove \Cref{claim:likely-most-groups-good}, we first show the probability any one group is bad is at most $0.01$ and then use \Cref{cor:direct-product-count} to show the probability that $0.02k$ groups are bad is $e^{-k/300}$.
\begin{proof}
    To formally apply \Cref{cor:direct-product-count} in this setting, we need to define all the parameters of the analyst game in \Cref{fig:analyst-game}.
    \begin{enumerate}
        \item The domain is $X^{\floor{n/k}}$.
        \item The distribution of each group is $\mcD_{\base} \coloneqq \mcD^{\floor{n/k}}$.
        \item The budget of each group is $b_{\base} = O(1/w_{\max})$.
        \item The queries $\Phi$ consist of all subsampling queries with range $\zo$.
        \item The cost of a subsampling query $\phi:X^w \to \zo$ is $\frac{w}{(\floor{n/k})^2}$.
        \item For each subsampling query $\psi:X^w \to R \subseteq \R$ where $w \leq w_{\max}$, we have a test query $\psi_{\phi}:X^{\floor{n/k}} \to \set{0,1}$ which takes as input $S_i$ are returns whether $S_i$ is bad for the query $\phi$. $\Psi$ consists of all such test queries.
    \end{enumerate}
    We'll show that these parameters are consistent with the median mechanism in \Cref{fig:median-mechanism}. For each $t \in [T]$, to execute the binary search, the mechanism makes $\ceil{\log_2(|R_t|)}$ queries per group, with each query have cost $\frac{w_t k^2}{n^2}$. Therefore, the total budget used in each group is
    \begin{equation*}
        \frac{\sum_{t \in [T]}\log_2(|R_t|) w_t}{(\floor{n/k})^2} \leq O\paren*{\lfrac{1}{w_{\max}}},
    \end{equation*}
    where the above inequality follows from the parameters set in \Cref{thm:median}. Therefore, the desired result will follow from \Cref{cor:direct-product-count} combined with showing $\AG(\mcD_{\base}, b_{\base}) \leq 0.01$.

    The statement $\AG(\mcD_{\base}, b_{\base}) \leq 0.01$ is equivalent to the following: For a single $\bS_i \sim \mcD^{\floor{n/k}}$, if the analyst gets to ask subsampling queries with total cost $O(1/w_{\max})$, the probability it can choose $\bpsi:X^w \to R \subseteq \R$ with $w \leq w_{\max}$ for which $\bS_i$ is bad is at most $0.01$. Given any such $\bpsi$, we define,
    \begin{equation*}
        \bpsi_{\min}(x_1, \ldots, x_{w}) = \Ind[\bpsi_t(x_1, \ldots, x_{w}) < r_{\min}] \quad\text{and}\quad \bpsi_{\max}(x_1, \ldots, x_{w_t}) = \Ind[\bpsi(x_1, \ldots, x_{w}) > r_{\max}]
    \end{equation*}
    where $r_{\min}$ and $r_{\max}$ are the smallest and largest approximate median of $\bpsi(\mcD)$ respectively. Then, by \Cref{def:approx-median},
    $\mu_{\bpsi_{\min}}(\mcD) \leq 0.4$ and $\mu_{\bpsi_{\max}}(\mcD) \leq 0.4$. However, for $\bS_i$ to be bad for $\bphi_t$, it must be the case that $\mu_{\bpsi_{\min}}(\bS_i) > 0.45$ or $\mu_{\bpsi_{\max}}(\bS_i) > 0.45$. Therefore, if $\bS_i$ is bad $\bphi_t$, it must be the case that 
    \begin{equation*}
        \max_{\bpsi \in \paren*{\bpsi_{\min}, \bpsi_{\max}}} \paren*{\error(\bpsi, \bS_i, \mcD)} \geq \Omega(1/w).
    \end{equation*}
    Using \Cref{thm:main-binary} and that the analyst is limited to asking queries with a total budget of $O(1/w_{\max})$,
    \begin{equation*}
        \Ex\bracket*{ \max_{\bpsi \in \paren*{\bpsi_{\min}, \bpsi_{\max}}} \set*{\error(\bpsi, \bS_i, \mcD)} } \leq O\paren*{\frac{1}{w_{\max}} + \frac{1}{\floor{n/k}}} \leq O(1/w_{\max}).
    \end{equation*}
    The desired result follows from $w \leq w_{\max}$ and Markov's inequality.
\end{proof}
We are now ready to prove the main result of this section, bounding the error of our median mechanism.
\begin{proof}[Proof of \Cref{thm:median}]
    For each $t \in T$, let $\bE_t$ be the indicator that the median mechanism fails to output an approximate median for $\bphi_t$. Then,
    \begin{align*}
        \Pr[\bE_t] &= \Pr[\bE_t, \text{at most $0.02k$ groups bad for $\bphi_t$}] + \Pr[\bE_t, \text{more than $0.02k$ groups bad for $\bphi_t$}] \\
        &\leq \log_2 |R_t| e^{-\Omega(k)} + e^{-\Omega(k)}. \tag{\Cref{claim:median-good-means-accurate} and \Cref{claim:likely-most-groups-good}}
    \end{align*}
    Then, by union bound, the probability the mechanism fails for any $t \in T$ is at most $T(\log_2 |R_{\max}| + 1)e^{-\Omega(k)}$. The desired result follows by how $k$ is set in \Cref{thm:median}.
\end{proof}

\section{Acknowledgments}
The author thanks Li-Yang Tan and Jonathan Ullman for their helpful discussions and feedback. He also very gratefully thanks the STOC and JACM reviewers for their helpful feedback. In particular, one of the anonymous JACM reviewers suggested the presentation in \Cref{subsec:ALKL-overview}.

Guy is supported by NSF awards 1942123, 2211237, and 2224246 and a Jane Street Graduate Research Fellowship.

\bibliographystyle{alpha}
\bibliography{ref}

\appendix
\section{The average leave-one-out KL stability of subsampling queries}
\label{sec:tightness-bad}

A key part of our mutual information bound is the new notion of stability we develop which generalizes Feldman and Steinke's notion of average leave-\emph{one}-out KL stability \cite{FS18}.
\begin{definition}[Average leave-one-out KL stability, restatement of \Cref{def:ALKL-intro}, \cite{FS18}]
    \label{def:ALOOKL}
    A randomized algorithm $\mcM: X^n \to Y$ is $\eps$-\emph{ALOOKL stable} if there is a randomized algorithm $\mcM': X^{n-1} \to Y$ such that, for all samples $S \in X^n$,
    \begin{equation*}
        \Ex_{\bi \sim [n]} \bracket*{\KL{\mcM(S)}{\mcM'(S_{-{\bi}})}} \leq \eps.
    \end{equation*}
\end{definition}
Feldman and Steinke show how to use ALOOKL stability to bound mutual information.
\begin{lemma}[Using ALOOKL stability to bound mutual information]
    \label{lem:alookl-to-mi}
    Let $\bS$ drawn from a product distribution over $X^n$ and, for each $t \in [T]$, draw $\by_t \sim \mcM^{\by_1, \ldots, \by_{t-1}}(\bS)$ where $\mcM^{\by_1, \ldots, \by_{t-1}}$ is $\eps$-ALOOKL stable. Then,
    \begin{equation*}
        I(\bS; (\by, \ldots, \by_T)) \leq nT\eps.
    \end{equation*}
\end{lemma}
The conference version of this paper \cite{B23STOC} used ALOOKL stability to obtain a weaker bound on the mutual information. A natural question is whether we needed to change ALOOKL stability to ALMOKL stability in order to obtain the improved mutual information bound of \Cref{thm:MI-intro}. Here, we show that is indeed the case, and that the conference version's bound is tight.
\begin{lemma}
    For any $n \geq 1000$, there is a $\phi:X^1 \to \zo$ which, when applied as a subsampling query to a sample of size $n$, is not $\eps$-ALOOKL stable for any $\eps \leq \frac{\ln n}{2n^2}$.
\end{lemma}
To recover \Cref{thm:MI-intro} (up to constants), we would need to $\eps \leq O(1/n^2)$.
\begin{proof}
    Let $X \coloneqq \zo$ and $\phi$ be the indicator function. We'll use the notation $\vec{0}_n$ to refer to the tuple containing $n$-many $0$s. Upon receiving the input $\vec{0}_{n-1}$, let $p$ be that probability that $\mcM'$ outputs $1$ (so that $\mcM'(\vec{0}_{n-1}) = \Ber(p)$). We consider two cases.

    \textbf{Case 1:} $p > \ln n/(2n^2)$. In this case, consider $S \coloneqq \vec{0}_n$. We'll show that,
    \begin{equation*}
        \Ex_{\bi \sim [n]} \bracket*{\KL{\phi(S)}{\mcM'(S_{-{\bi}})}} \geq p > \frac{\ln n}{2n^2}.
    \end{equation*}
    This is the easier case. No matter which element is removed, $S_{-\bi} = \vec{0}_{n-1}$, so $\mcM'(S_{-{\bi}}) = \Ber(p)$. On the other hand, $\phi(S) = \Ber(0)$. Therefore,
    \begin{equation*}
        \Ex_{\bi \sim [n]} \bracket*{\KL{\phi(S)}{\mcM'(S_{-{\bi}})}} = \KL{\Ber(0)}{\Ber(p)} = -\ln(1-p) \geq p.
    \end{equation*}

    \textbf{Case 2:} $p \leq \ln n/(2n^2)$. In this case, consider the sample containing a single $1$, $S \coloneqq \vec{0}_{n-1} \circ 1$. Note that $\phi(S) = \Ber(1/n)$, and there is a $1/n$ chance that $S_{-\bi} =\vec{0}_{n-1} $, so
    \begin{equation*}
        \Ex_{\bi \sim [n]} \bracket*{\KL{\phi(S)}{\mcM'(S_{-{\bi}})}} \geq 1/n \cdot  \KL{\Ber(1/n)}{\Ber(p)}.
    \end{equation*}
    The above KL divergence is larger the further $p$ is from $1/n$. Since $\ln n/(2n^2) < 1/n$, we can lower bound,
    \begin{align*}
        \Ex_{\bi \sim [n]} \bracket*{\KL{\phi(S)}{\mcM'(S_{-{\bi}})}} &\geq 1/n \cdot  \KL{\Ber(1/n)}{\Ber(\ln n/(2n^2))} \\
        &= 1/n \cdot \paren*{1/n \cdot \ln\paren*{\frac{1/n}{\ln n/(2n^2)}} + (1 - 1/n) \cdot \ln\paren*{\frac{1-1/n}{1-\ln n/(2n^2)}}}\\
        &\geq 1/n \cdot \paren*{1/n \cdot \ln\paren*{\frac{2n}{\ln n}} + 1 \cdot \ln\paren*{1-1/n}}
    \end{align*}
    where in the last inequality, we used that the second term was negative so we are free to increase its magnitude while maintaining a valid lower bound. Here we use that when $n \geq 1000$, $\ln(1-1/n) \geq -1.001 \cdot 1/n$ to bound
    \begin{align*}
        \Ex_{\bi \sim [n]} \bracket*{\KL{\phi(S)}{\mcM'(S_{-{\bi}})}} &\geq 1/n^2 \cdot \paren*{\ln(n) - \ln \ln (n) -1.001}
    \end{align*}
    which, using $n \geq 1000$, is strictly more than $\ln(n)/(2n^2)$.
\end{proof}

\end{document}